\documentclass[twoside,11pt]{article}

\usepackage{jmlr2e}
\usepackage[ruled,vlined]{algorithm2e}
\usepackage{microtype}
\usepackage{graphicx}
\usepackage{amsmath}
\usepackage{amssymb}
\usepackage{enumerate}
\usepackage{enumitem}

\usepackage{subfigure}
\usepackage{booktabs} 
\newcommand{\E}{\mathbb{E}}

\newcommand{\p}{\mathbb{P}}

\newtheorem{assumption}{Assumption}
\newtheorem{claim}{Claim}

\usepackage{hyperref}

\jmlrheading{21}{2020}{1-29}{06/20}{}{cayci20a}{Semih Cayci, Atilla Eryilmaz and R. Srikant}

\ShortHeadings{Continuous-Time Multi-Armed Bandits with Controlled Restarts}{Cayci, Eryilmaz and Srikant}
\firstpageno{1}

\begin{document}

\title{Continuous-Time Multi-Armed Bandits with Controlled Restarts}

\author{\name Semih Cayci \email cayci.1@osu.edu \\
       \addr Department of Electrical and Computer Engineering\\
       The Ohio State University\\
       Columbus, OH 43210, USA
       \AND
       \name Atilla Eryilmaz \email eryilmaz.2@osu.edu \\
       \addr Department of Electrical and Computer Engineering\\
       The Ohio State University\\
       Columbus, OH 43210, USA
       \AND
       \name R. Srikant \email rsrikant@illinois.edu \\
       \addr Department of Electrical and Computer Engineering\\
       University of Illinois at Urbana-Champaign\\
       Urbana, IL 61801, USA}

\editor{}

\maketitle

\begin{abstract}
Time-constrained decision processes have been ubiquitous in many fundamental applications in physics, biology and computer science. Recently, restart strategies have gained significant attention for boosting the efficiency of time-constrained processes by expediting the completion times. In this work, we investigate the bandit problem with controlled restarts for time-constrained decision processes, and develop provably good learning algorithms. In particular, we consider a bandit setting where each decision takes a random completion time, and yields a random and correlated reward at the end, with unknown values at the time of decision. The goal of the decision-maker is to maximize the expected total reward subject to a time constraint $\tau$. As an additional control, we allow the decision-maker to interrupt an ongoing task and forgo its reward for a potentially more rewarding alternative. For this problem, we develop efficient online learning algorithms with $O(\log(\tau))$ and $O(\sqrt{\tau\log(\tau)})$ regret in a finite and continuous action space of restart strategies, respectively. We demonstrate an applicability of our algorithm by using it to boost the performance of SAT solvers.

\end{abstract}


\section{Introduction}
Time-constrained processes, which continue until the total time spent exceeds a given time horizon, have long been a focal point of scientific interest as a consequence of their universal applicability in a broad class of disciplines including physics, biochemistry and computer science \citep{redner2001guide, condamin2007first}. Recently, it has been shown that any time-constrained process can employ controlled restarts with the goal of expediting the completion times, thus increasing the time-efficiency of stochastic systems \citep{pal2017first}. Consequently, restart strategies have attracted significant attention to boost the time-efficiency of stochastic systems in various contexts. They have been extensively used to study diffusion mechanics \citep{evans2011diffusion, pal2016diffusion}, target search applications \citep{kusmierz2014first, eliazar2007searching}, catalyzing biochemical reactions \citep{rotbart2015michaelis}, throughput maximization \citep{asmussen2008asymptotic}, and run-times of randomized algorithms \citep{hoos2004stochastic, luby1993optimal, selman1994noise}. In particular, they have been widely used as a tool for the optimization of randomized algorithms that employ stochastic local search methods whose running times exhibit heavy-tailed behavior \citep{luby1993optimal, gomes1998boosting}.

In this paper, we investigate the exploration-exploitation problem in the context of time-constrained decision processes under controlled restarts, and develop online learning algorithms to learn the optimal action and restart strategy in a knapsack bandit framework. This learning problem has unique dynamics: the cumulative reward function is a controlled and stopped random walk with potentially heavy-tailed increments, and the restart mechanism leads to right-censored feedback, which imposes a specific information structure. In order to design efficient learning algorithms that fully capture these dynamics, we incorporate new design and analysis techniques from bandit theory, renewal theory and statistical estimation.

As a fundamental application of this framework, we investigate the learning problem for boosting the stochastic local search methods under restart strategies both theoretically and empirically. In particular, our empirical studies on various $k$-SAT databases show that the learning algorithms introduced in this paper can be efficiently used as a \textit{meta-algorithm} to boost the time-efficiency of SAT solvers.

\subsection{Related Work}
Bandit problem with knapsack constraints and its variants have been studied extensively in the last decade. In \citep{gyorgy2007continuous, badanidiyuru2013bandits, combes2015bandits, xia2015thompson, tran2012knapsack, cayci2020budget}, the problem was considered in a stochastic setting. This basic setting was extended to linear contextual setting \citep{agrawal2016linear}, combinatorial semi-bandit setting \citep{sankararaman2017combinatorial}, adversarial bandit setting \citep{immorlica2019adversarial}. For further details in this branch of bandit theory, we refer to \citep{slivkins2019introduction}. As a substantial difference, these works do not incorporate a restart or cancellation mechanism into the learning problem.

The restart mechanisms have been popular particularly in the context of boosting the Las Vegas algorithms. The pioneering work in this branch is \citep{luby1993optimal}, where a minimax restart strategy is proposed to achieve optimal expected run-time up to a logarithmic factor. In \citep{gagliolo2007learning}, a hybrid learning strategy between Luby scheme and a fixed restart strategy is developed in the adversarial setting. In \citep{streeter2009online}, an algorithm portfolio design methodology that employs a restart strategy is developed as an extension of the Exp3 Algorithm in the non-stochastic bandit setting, and polynomial regret is shown with respect to a constant factor approximation of the optimal strategy. These works are designed in a non-stochastic setting for worst-case performance metric, thus they do not incorporate the statistical structure of the problem that we consider in this work. 

The most related paper in the budget-constrained bandit setting is \citep{cayci2019learning}. This current paper provides improves and extends \citep{cayci2019learning} in multiple directions. We propose algorithms that work on continuous set of restart times, incorporate empirical estimates to eliminate the necessity of the prior knowledge on the moments. We also investigate the impact of correlation between the completion time and reward on the behavior of the restart times.

In light of the above related work, our main contributions in this paper are as follows:
\begin{itemize}
    \item This work intends to provide a principled approach to continuous-time exploration-exploitation problems that require restart strategies for optimal time efficiency. In order to achieve this, we study and explain the impact of restart strategies in a general knapsack bandit setting that includes potentially heavy-tailed and correlated time-reward pairs for each arm. From a technical perspective, the design and regret analysis of learning algorithms requires tools from renewal theory and stochastic control, as well as novel concentration inequalities for rate estimation, which can be useful in other bandit problems as well.
    \item For a finite set of restart times, we propose algorithms that use variance estimates and the information structure that stems from the right-censored feedback so as to achieve order-optimal regret bounds with no prior knowledge. In particular, we propose:
    \begin{itemize}
        \item The UCB-RB Algorithm based on empirical rate estimation to achieve particularly good performance for small restart times,
        \item The UCB-RM Algorithm based on median-boosted rate estimation to achieve good performance in a very general setting of large (potentially infinite) restart times at the expense of degraded performance for small restart times.
    \end{itemize}
    These algorithms utilize empirical estimates as a surrogate for unknown parameters, hence they provably achieve tight regret bounds with no prior information on the parameters. 
    \item For a continuous decision set for restart strategies, we propose an algorithm called UCB-RC that achieves $O(\sqrt{\tau\log(\tau)})$ regret.
    \item We evaluate the performance of the learning algorithm developed in this paper for the $k$-SAT problem on a SATLIB dataset by using WalkSAT, and empirically show its efficiency by comparing the results with benchmark strategies including Luby strategy.
\end{itemize}

\subsection{Notation}
In this subsection, we define some notation that will be used throughout the paper. $\mathbb{I}$ denotes the indicator function. For any $k\in\mathbb{Z}_+$, $[k]$ denotes the set of integers from 1 to $k$, i.e., $[k]=\{1,2,\ldots,k\}$. For $x\in\mathbb{R}$, $(x)_+=\max\{0,x\}$. For two real numbers $x,y$, we define $x\wedge y = \min\{x, y\}$ and $x\vee y = \max\{x, y\}$.

In the following section, we describe the problem setting in detail.

\section{System Setup}\label{sec:problem}
We consider a time-constrained decision process with a given time horizon $\tau > 0$. The increments of the process are controlled as follows: if arm $k\in[K]$ is selected at the $n$-th trial, the random completion time is $X_{k,n}$, that is independent across $k$ and independent over $n$ for each $k$ with the following weak conditions:
\begin{align}\label{eqn:moment-assn}
    \E[X_{k,1}^p] &< \infty,~p \geq 4,\\
    X_{k,n}&>0,~a.s.
\end{align}
\noindent Note that $X_{k,n}$ can potentially be a heavy-tailed random variable. Pulling arm $k$ at trial $n$ yields a reward $R_{k,n}$ at the end. Note that we allow $R_{k,n}$ to be possibly correlated with $X_{k,n}$, and for simplicity, we assume that $R_{k,n}\in[0,1]$ almost surely.

In our framework, the controller has the option to interrupt an ongoing task and restart it with a potentially different arm. Namely, for a set of \textit{restart} (or \textit{cutoff}) times $\mathbb{T}\subset [0,\infty]$, the controller may prefer to restart the task from its initial state after time $t\in\mathbb{T}$ at the loss of the ultimate rewards and possible additional time-varying cost of restarting the process. The additional cost is included, since in some applications, it may take a \textit{resetting time} to return to the initial state after a reset decision \citep{evans2011diffusion}. We model the resetting time as $C_k(t)\in[0,t]$, which is a deterministic function of $t\in\mathbb{T}$ known by the controller for simplicity. Note that the design and analysis presented in this paper can be extended to random resetting time processes in a straightforward way.  

If the $n$-th trial of arm $k$ is restarted at time $t\in\mathbb{T}$, then the resulting total time and reward are as follows:
\begin{align}
    \label{eqn:ukl} U_{k,n}(t) &= \min\{X_{k,n}, t\}+C_k(t)\mathbb{I}\{X_{k,n}>t\},\\
    V_{k,n}(t) &= R_{k,n}\mathbb{I}\{X_{k,n}\leq t\},
\end{align}
\noindent Note that the feedback $(U_{k,n}(t),V_{k,n}(t))$ in this case is right-censored.

Incorporating the restart mechanism, the control of the decision-maker consists of two decisions at $n$-th trial: $$\pi_{n} = (I_n, \nu_n)\in[K]\times\mathbb{T}.$$ Here, $I_n\in[K]$ denotes the arm decision, and $\nu_n\in\mathbb{T}$ denotes the restart time decision. Under a policy $\pi=\{\pi_n:n\geq 1\}$, let $$\mathcal{F}_n^\pi=\sigma\big(\{(U_{I_i,i}(\nu_i),V_{I_i,i}(\nu_i):i=1,2,\ldots,n\}\big),$$ be the history until $n$-th trial. We call a policy $\pi$ admissible if $\pi_{n+1}\in\mathcal{F}_n^\pi$ for all $n\geq 1$.

For a given time horizon $\tau > 0$, the objective of the decision-maker is to maximize the expected cumulative reward collected in the time interval $[0,\tau]$. Specifically, letting
\begin{equation}
    N_\pi(\tau) = \inf\{n:\sum_{i=1}^nU_{I_i,i}(\nu_i)>\tau\},
\end{equation}
\noindent be the number of pulls under policy $\pi$, the cumulative reward under $\pi$ is defined as follows:
\begin{equation}
    {\tt REW}_\pi(\tau) = \sum_{n=1}^{N_\pi(\tau)}V_{I_n,n}(\nu_n).
\end{equation}
The decision-maker attempts to achieve the optimal reward:
\begin{equation}\label{eqn:opt-reward}
    {\tt OPT}(\tau) = \max\limits_{\pi\in\Pi}~\E[{\tt REW}_\pi(\tau)],
\end{equation}
\noindent where $\Pi$ is the set of all admissible policies. Equivalently, it aims to minimize the expected regret:
\begin{equation}
    {\tt REG}_\pi(\tau) = {\tt OPT}(\tau)-\E[{\tt REW}_\pi(\tau)].
\end{equation}

In the following section, we provide a notable example that falls within the scope of this framework.

\section{Application: Boosting the Local Search for $k$-SAT}
Boolean satisfiability (SAT) is a canonical NP-complete problem \citep{hoos2004stochastic, arora2009computational}. As a consequence, the design of efficient SAT solvers has been an important long standing challenge, and an annual SAT Competition is held for SAT solvers \citep{heule2019sat}. In these competitions, the objective of a SAT solver is to solve as many problem instances as possible within a given time interval $[0,\tau]$, which defines a time-constrained decision process. As a result of the completion time distributions, the restart strategies have been an essential part of the SAT solvers \citep{luby1993optimal, selman1994noise}.

The WalkSAT Algorithm is one of the most fundamental SAT solvers \citep{papadimitriou1991selecting, hoos2004stochastic}. In its most basic form, it employs a randomized local search methodology to converge to a valid assignment of the problem instance: at each trial, it randomly chooses an assignment, then a chosen variable is flipped, and its new value is kept if the number of satisfied clauses increases \citep{papadimitriou1991selecting}. Experimental studies indicate that the completion time of $n$-th trial, $X_{1,n}$, is a heavy-tailed random variable with potentially infinite mean \citep{gomes1998boosting}. Moreover, the successive trials of the same instance have iid completion times as a result of the random initialization. As such, the problem of maximizing the number of solved problem instances within $[0,\tau]$ can be formulated as \eqref{eqn:opt-reward} with $R_{1,n}=1$ for all $n$, and the online learning algorithms in this paper can be used as meta-algorithms to boost the performance of the WalkSAT.


We will use this problem for testing the benefits of our design for boosting SAT solvers. 

\section{Near-Optimal Policies}
The control problem described in \eqref{eqn:opt-reward} is a variant of the well-known unknown stochastic knapsack problem \citep{kleinberg2006algorithm}. In the literature, there are similar problems that are known to be PSPACE-hard \citep{papadimitriou1999complexity, badanidiyuru2013bandits}. Therefore, we need tractable algorithms to approximate the optimal policy. In this section, we will propose a simple policy for the problem introduced in Section \ref{sec:problem}, and prove its efficiency by using the theory of renewal processes and stopping times.

In the following, we will first provide an upper bound for ${\tt OPT}(\tau)$. The quantity of interest will be the (renewal) reward rate, which is defined next.
\begin{definition}[Reward Rate]
    For a decision $(k, t)$, the renewal reward rate is defined as follows: 
    \begin{equation}\label{eqn:rew-rate}
        r_k(t) = \frac{\mathbb{E}[R_{k,1}\mathbb{I}\{X_{k,1}\leq t\}]}{\mathbb{E}[(X_{k,1}\wedge t)+C_{k}(t)]} = \frac{\E[V_{k,1}(t)]}{\E[U_{k,1}(t)]}.
    \end{equation}
\end{definition}
The reward rate $r_k(t)$ is the growth rate of the expected total reward over time if the controller persistently chooses the action $(k,t)$. In other words, as a consequence of the elementary renewal theorem \citep{gut2009stopped}, the reward of the static policy that persistently makes a decision $(k,t)$ is $\mathbb{E}[{\tt REW}_\pi(\tau)] = r_k(t)\cdot\tau+o(\tau).$ In the following, we provide an upper bound based on $r_k(t)$ and the time horizon $\tau$.
\begin{proposition}[Upper bound for {\tt OPT}]\label{prop:rew-opt}
Let the optimal reward rate be defined as follows:
\begin{equation}
    r^*=\arg\max\limits_{(k,t)}~r_k(t).
\end{equation}
If there exists a $p_0>2$ and $u<\infty$ such that $\mathbb{E}[(X_{k,1})^{p_0}]\leq u$ holds for all $k\in[K]$, then we have the following upper bound for ${\tt OPT}$:
\begin{equation}
    {\tt OPT}(\tau) \leq r^*\big(\tau+\Phi(u)\big),
\end{equation}
\noindent for any $\tau>0$ where $\Phi(u)$ is a constant that is independent of $\tau$.
\end{proposition}

\begin{proof}
The proof generalizes the optimality gap results in \citep{xia2015thompson, cayci2020budget}. Under any admissible policy $\pi\in\Pi$, an extension of Wald's identity yields the following upper bound:
\begin{equation*}
    \mathbb{E}[{\tt REW}_\pi(\tau)] \leq r^*\mathbb{E}[S_{N_\pi(\tau)}],
\end{equation*}
\noindent where $N_\pi(\tau)$ is the first-passage time of the controlled random walk under $\pi$. The excess over the boundary, $\mathbb{E}[S_{N_\pi(\tau)}]-\tau$, is known to be $o(\tau)$ for simple random walks by the elementary renewal theorem \citep{gut2009stopped}. For controlled random walks, \cite{lalley1986control} shows that $\mathbb{E}[S_{N_\pi(\tau)}]-\tau=O(u)$, which is independent of $\tau$, if $\mathbb{E}[X_{k,1}^{p_0}]\leq u<\infty$ holds for all $k\in[K]$ for some $p_0>2$. Thus, under the moment assumption \eqref{eqn:moment-assn}, we have $\mathbb{E}[S_{N_\pi(\tau)}]-\tau=O(1)$ as $\tau\rightarrow\infty$.
\end{proof}

From the discussion above, it is natural to consider an algorithm that optimizes $r_k(t)$ among all decisions $(k,t)$ as an approximation of the optimal policy. Accordingly, the optimal static policy, denoted as $\pi^{\tt st}$, is given in Algorithm \ref{alg:opt-st}.
\begin{algorithm}
$n=0$\;
$S_n=0$\;
\While{$S_n \leq \tau$}{
$(k^*,t^*)=\arg\max\limits_{(k,t)}~r_k(t)$;\hfill // Maximize \eqref{eqn:rew-rate}\\
\eIf{$X_{k^*,n} \leq t$}{
$S_{n+1}=S_n+X_{k^*,n}$\;
Obtain reward $R_{k^*,n}$\;
} {
$S_{n+1} = S_n+t^*+C_{k^*}(t^*)$;\hfill //Restart time\\
}
$n=n+1$;\hfill //New epoch starts\\
}
\caption{Optimal Static Algorithm $\pi^{\tt st}$}\label{alg:opt-st}
\end{algorithm}

The performance analysis of $\pi^{\tt st}$ is fairly straightforward since the random process it induces is a simple random walk.
\begin{proposition}\label{prop:rew-st}
The reward under $\pi^{\tt st}$ is bounded as follows:
\begin{equation}
    r^*\tau \leq \mathbb{E}[{\tt REW}_{\pi^{\tt st}}(\tau)] \leq r^*\Big(\tau + \frac{\mathbb{E}[Y_*^2]}{\big(\mathbb{E}[Y_*]\big)^2}\Big), 
\end{equation}
\noindent where $Y_* = U_{k^*,1}(t^*)$ is the completion time of an epoch under $\pi^{\tt st}$, and $(k^*, t^*)$ is defined in Algorithm \ref{alg:opt-st}.
\end{proposition}
The proof of Proposition \ref{prop:rew-st} follows from Lorden's inequality for renewal processes \citep{asmussen2008applied}.

As a consequence of Proposition \ref{prop:rew-opt} and Proposition \ref{prop:rew-st}, the optimality gap of the static policy is bounded for all $\tau>0$.
\begin{corollary}[Optimality Gap of $\pi^{\tt st}$]
For any $\tau>0$, the optimality gap of $\pi^{\tt st}$ is bounded:
\begin{equation}
    {\tt OPT}(\tau)-\mathbb{E}[{\tt REW}_{\pi^{\tt st}}(\tau)] \leq r^*\Phi(u),
\end{equation}
\noindent where $\Phi(u)$ is the constant in Prop. \ref{prop:rew-opt}.
\end{corollary}

In this section, we observed that the reward rate $r_k(t)$ is the dominant component of the cumulative reward. In the next section, we will analyze the behavior of $r_k(t)$ with respect to the restart time $t$.

\section{Finiteness of Optimal Restart Times}
\label{sec:opt-restart}
For any given arm $k\in[K]$ and any set of restart times $\mathbb{T}$, let the optimal restart time be defined as follows:
\begin{equation}
    t_k^* \in \arg\max_{t\in\mathbb{T}}~r_k(t).
    \label{eqn:opt-restart-time}
\end{equation}

If the completion time $X_{k,n}$ and reward $R_{k,n}$ are independent, then it was shown in \citep{cayci2019learning} that it is optimal to restart a cycle at a finite time if the following condition holds:
\begin{equation}\label{eqn:opt-int-ind}
    \mathbb{E}[X_{1,k}-t|X_{1,k}>t] > \mathbb{E}[X_{1,k}]
\end{equation}
\noindent for $t\in \mathbb{T}\setminus\{\infty\}$, and it is noted that all heavy-tailed and some light-tailed completion time distributions satisfy this condition. If $X_{k,n}$ and $R_{k,n}$ are correlated, this is no longer true and the situation is more complicated. In the following, we extend this result to correlated $(X_{k,n},R_{k,n})$ pairs, and investigate the effect of the return time $C_{k}(t)$.

\begin{theorem}[Optimal Restart Time]\label{thm:opt-int}
    For a given arm $\big(X_{k,n},R_{k,n},C_{k}(t)\big)$, we have $t_k^*<\infty$ if and only if the following holds:
    \begin{equation}
        \frac{\mathbb{E}\big[R_{k,n}\big|X_{k,n}>t\big]}{\mathbb{E}\big[X_{k,n}-\big(t+C_{k}(t)\big)\big|X_{k,n}>t\big]} < \frac{\E[R_{k,1}]}{\E[X_{k,1}]},
    \label{eqn:opt-restart}
    \end{equation}
    \noindent for some $t > 0$.
\end{theorem}
\noindent The proof follows from showing \eqref{eqn:opt-restart} is equivalent to $r_k(t) > \E[R_{k,1}]/\E[X_{k,1}]$ for any $t >0$.

Interpretation of Theorem \ref{thm:opt-int} is as follows: for any restart time $t \in \mathbb{T}\setminus\{\infty\}$, if the reward rate of waiting until the completion (i.e., \emph{the residual reward rate}) is lower than the reward rate of a new trial, then it is optimal to restart.

\begin{remark}\normalfont
Note that \eqref{eqn:opt-int-ind} is a special case of Theorem \ref{thm:opt-int} with immediate returns, i.e., $C_{k}(t)=0$, and $X_{k,n}$ and $R_{k,n}$ are independent.
\end{remark}

\begin{remark}[The effect of correlation]\normalfont The correlation between $X_{k,n}$ and $R_{k,n}$ has a substantial impact on whether the optimal restart times are finite or not. As an example, let $X_{k,n}\sim Pareto(1, \alpha)$ for some $\alpha \in (1,2)$, and $C_k(t) = 0$ for all $t$.
\begin{enumerate}
    \item If $X_{k,n}$ and $R_{k,n}$ are independent, then \eqref{eqn:opt-int-ind} holds and $t_k^* < \infty$, i.e., it is optimal to restart after a finite time.
    \item If $R_{k,n}=\omega X_{k,n}^\gamma$ for some $\omega>0$ and $\gamma \geq 1$, then it is optimal to wait until the end of the task, i.e., $t_k^* = \infty$.
    \item If $R_{k,n}=\omega X_{k,n}^\gamma$ for $\omega>0$ and $\gamma<1$, then we have $t_k^*<\infty$.
\end{enumerate}
The impact of the correlation between $X_{k,n}$ and $R_{k,n}$ on the behavior of optimal restart time is illustrated in Figure \ref{fig:pareto-correlation}.
\begin{figure}
    \centering
    \includegraphics[scale=0.6]{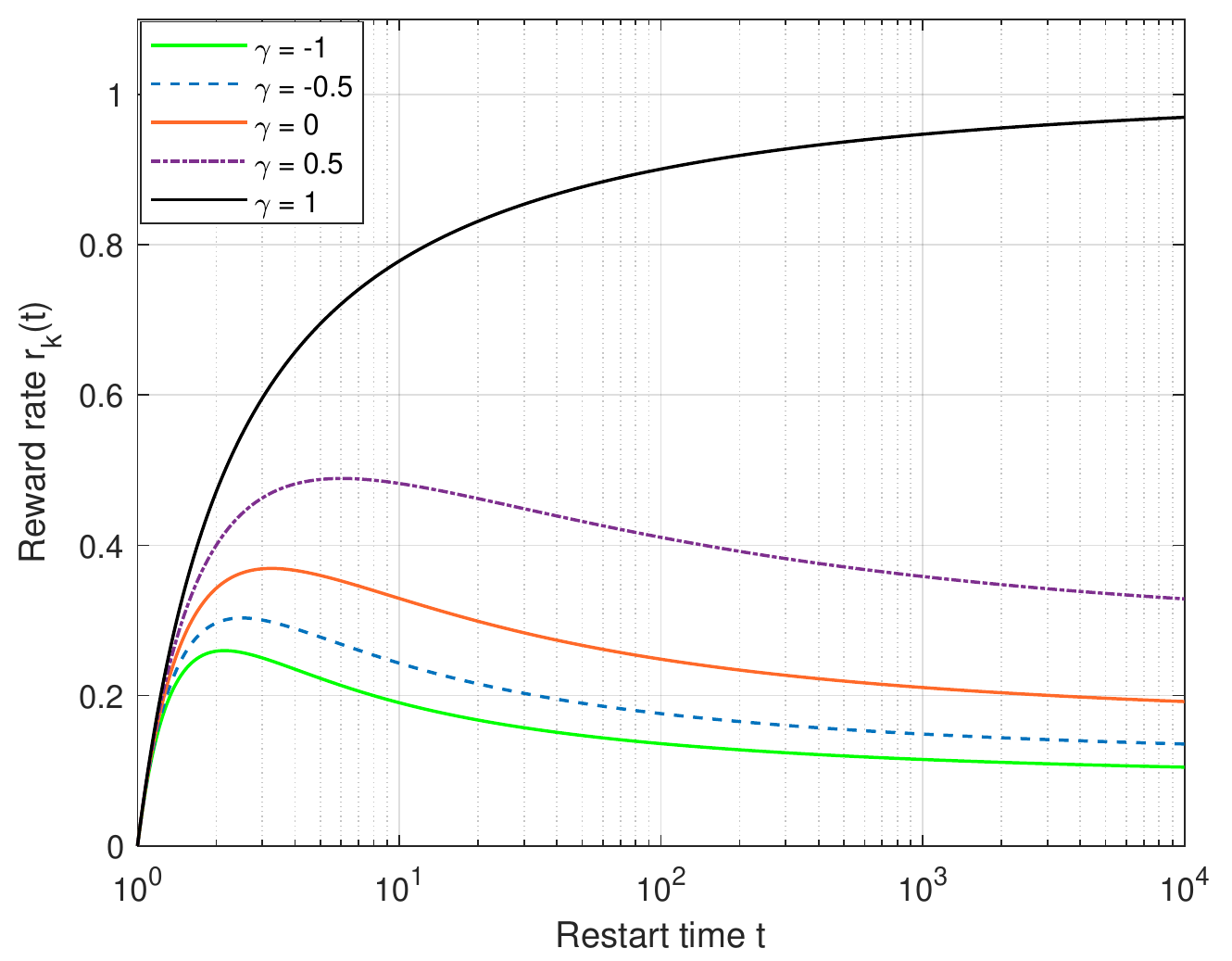}
    \caption{Impact of correlation between $X_{k,n}$ and $R_{k,n}$ on the optimal restart time for $X_{k,n}\sim Pareto(1,1.2)$ and $\gamma\in[-1,1]$. Positive correlation between the completion time and reward leads to higher restart times, and waiting until the completion of every task is optimal for $\gamma\geq 1$ since waiting becomes more rewarding.}
    \label{fig:pareto-correlation}
\end{figure}
\end{remark}

In the next section, we develop online learning algorithms for the problem, and present the regret bounds.

\section{Online Learning Algorithms for Controlled Restarts}
In this section, we develop online learning algorithms with provably good performance guarantees. The right-censored nature of the feedback due to the restart mechanism imposes an interesting information structure to this problem. We first describe the nature of this information structure.

\subsection{Right-Censored Feedback and Information Structure}\label{subsec:info-structure}
Recall that the feedback we obtain for a decision of $(k, t)$ is the pair of right-censored random variables $(U_{k,n}(t), V_{k,n}(t))$ as in \eqref{eqn:ukl}. As a consequence, for any $t^\prime \leq t$, we have the following:
\begin{align*}
    U_{k,n}(t^\prime) &\in \sigma\big(U_{k,n}(t)\big),\\
    V_{k,n}(t^\prime) &\in \sigma\big(V_{k,n}(t)\big).
\end{align*}
\noindent In other words, the feedback from a restart time decision $t>0$ can be faithfully used as a feedback for another restart time decision $t^\prime \leq t$. This implies that the information gain by a large $t\in\mathbb{T}$ is larger compared to $t^\prime\leq t$.

\subsection{Finite Set of Restart Times: {UCB-RB} and UCB-RM}
We consider a finite $\mathbb{T}=\{t_1,t_2,\ldots,t_L\}$ such that 
\begin{equation}
    t_1< t_2 < \ldots < t_L \leq \infty.
\end{equation}
Throughout the paper, we assume that any action set $\mathbb{T}$ satisfies the following assumption:
\begin{assumption}
Given a decision set $\mathbb{T}$, there exists $\epsilon, \mu_*>0$ that satisfies the following: $$\E[\min\{X_{k,1},t_1\}] \geq \mu_*,$$ and $$\p(X_{k,1}\leq t_1|R_{k,1}=\rho) \geq \epsilon,~\forall \rho \in[0,1],$$ for all $k\in[K]$, where $t_1=\min \mathbb{T}$.
\label{assn:action-set}
\end{assumption}
\noindent Note that Assumption \ref{assn:action-set} is a simple technical condition that ensures efficient estimation of $r_k(t)$ from the samples of $(X_{k,n}, R_{k,n})$ for all $t\in\mathbb{T}$.

In order to capture the benefits of the information structure, for arm $k$ and restart time decision $t_l$, let $$\mathcal{I}_{k,l}(n) = \{i\leq n:\pi_n = (k, t_l)\}.$$ Then, the available feedback for a decision $(k, t_l)$ is as follows: $$\mathcal{I}_{k,l}^*(n) = \bigcup\limits_{l^\prime \geq l}\mathcal{I}_{k,l}(n).$$ The size of $\mathcal{I}_{k,l}^*(n)$, i.e., the number of samples available for $(k,l)$ is defined as follows:
\begin{equation}
    T_{k,l}^*(n) = |\mathcal{I}_{k,l}^*(n)| = \sum_{l^\prime\geq l}T_{k,l^\prime}(n),
\end{equation}
From above, it is observed that the information structure increases the number of samples substantially for each decision, i.e., $T_{k,l}^*(n) \geq T_{k,l}(n)$ for all $k,l$.

The radius of the action set $\mathbb{T}$ has a crucial impact in algorithm design and performance, depending on the tail distributions of the completion times. In the following, we propose two algorithms for small and large (potentially infinite) $t_L$, and compare their characteristics.
\subsubsection{UCB-RB Algorithm}
The analysis in Section \ref{sec:opt-restart} indicates that in many cases, the optimal restart time is finite for arm $k$, i.e., $t_k^*<\infty$. In such cases, the action set $\mathbb{T}$ is localized around the potential optimal restart times, i.e., $t_L$ has a finite value. As a direct consequence of this observation, the support set of the completion times $U_{k,n}(t)$ is small for all $t\in\mathbb{T}$, which enables the use of fast estimation techniques. Below, we propose an algorithm for this setting based on empirical Bernstein inequality inspired by the UCB-B2 Algorithm in the classical stochastic bandits with knapsacks setting \citep{cayci2020budget}.

The UCB-RB Algorithm is based on empirical estimation, and inherently assumes that $t_L<\infty$. For an index set $S\subset \mathbb{Z}_+$, let the empirical mean $\widehat{\E}_S$ and empirical variance $\mathbb{V}_S$ of a random sequence $\{Y_i:i\geq 1\}$ be defined as follows:
\begin{align}
    \widehat{\mathbb{E}}_S(Y) &= \frac{1}{|S|}\sum_{i\in S}Y_i,\label{eqn:emp-mean}\\
    \mathbb{V}_S(Y) &= \frac{1}{|S|}\sum_{i\in S}\Big(Y_i-\widehat{\mathbb{E}}_S[Y]\Big)^2.\label{eqn:emp-var}
\end{align}

For any $(k,t_l)$ pair, let:
\begin{equation}\label{eqn:emp-rate}
    \widehat{r}_{k,l,n} = \frac{\widehat{\E}_{\mathcal{I}_{k,l}^*(n)}[V_{k}(t_l)]}{\widehat{\E}_{\mathcal{I}_{k,l}^*(n)}[U_{k}(t_l)]},    
\end{equation}
\noindent be the empirical reward rate. For $\alpha > 2$ and $\beta \in (0,1)$, let
\begin{equation}\label{eqn:emp-radius}
    c_{k,l,n} = \frac{(\beta+1)^2}{1-\beta}\frac{\eta_{k,l,n}+\widehat{r}_{k,l,n}\epsilon_{k,l,n}}{\widehat{\E}_{\mathcal{I}_{k,l}^*(n)}[U_{k}(t_l)]},
\end{equation}
for the confidence radii
\begin{align*}
    \epsilon_{k,l,n} &= \frac{3t_l\log(n^\alpha)}{T_{k,l}^*(n)}+\sqrt{\frac{2\mathbb{V}_{\mathcal{I}_{k,l}^*(n)}(U_k(t_l))\log(n^\alpha)}{T_{k,l}^*(n)}},\\
    \eta_{k,l,n} &= \frac{3\log(n^\alpha)}{T_{k,l}^*(n)}+\sqrt{\frac{2\mathbb{V}_{\mathcal{I}_{k,l}^*(n)}(V_k(t_l))\log(n^\alpha)}{T_{k,l}^*(n)}}.
\end{align*}

Then, the controller under UCB-RB makes a decision at $(n+1)$-th stage as follows: $$(I_{n+1},\nu_{n+1})\in\underset{(k,t_l)\in[K]\times\mathbb{T}}{\arg\max}~\big\{\widehat{r}_{k,l,n}+c_{k,l,n}\big\}.$$ The {UCB-RB} Algorithm is defined in detail in Algorithm \ref{alg:ucb-r}.
\begin{algorithm}
$n=1, i=1$\;
$S_0=0$\;
// Begin initialization for {\tt init} trials\\
\While{$i \leq {\tt init} ~\&~ S_{n-1}\leq B$}{
    \For{$k\in[K],l\in[L]$}{
        $(I_n,\nu_n) = (k,t_l)$\;
        $S_n=U_{k,n}(t_l)$\;
        Obtain reward $V_{k,n}(t_l)$\;
        n=n+1\;
    }
    i=i+1\;
}
\While{$S_{n-1} \leq B$}{
Compute $\widehat{r}_{k,l,n-1}$ by \eqref{eqn:emp-rate} and $c_{k,l,n-1}$ by \eqref{eqn:emp-radius}.\\
$(I_{n},\nu_{n})=\arg\max\limits_{(k,t_l)}~\{\widehat{r}_{k,l,{n-1}}+c_{k,l,{n-1}}\}$\\
\eIf{$X_{I_{n},n} \leq \nu_{n}$}{
$S_{n}=S_{n-1}+X_{I_{n},n}$\;
Obtain reward $R_{I_{n},n}$\;
} {
$S_{n} = S_{n-1}+\nu_{n}+C_{k_{n}}(\nu_{n})$;\\
}
Update the estimates for $l:t_l \leq \nu_{n}$\;
Add $n\mapsto \mathcal{I}_{k,l}^*(n)$ for all $l:t_l\leq \nu_{n}$\;
$n=n+1$\;
}
\caption{{UCB-RB} Algorithm $\pi^{\tt B}(\mathbb{T})$}\label{alg:ucb-r}
\end{algorithm}
\noindent Note that the information structure that stems from the right-censored feedback is utilized in {UCB-RB}. As we will see in the performance analysis, this information structure leads to substantial improvements in the performance.

In the following, we provide problem-dependent regret upper bounds for the UCB-RB Algorithm. 



\begin{theorem}[Regret Upper Bound for {UCB-RB}]\label{thm:ucb-rb}
    For any arm $k$ and restart time $t_l$, let $\Delta_{k,l} = r_*-r_k(t_l)$ and
    \begin{align}
        \mathbb{C}_{k,l} &= \mathbb{C}^2\big(U_{k,1}(t_l)\big)+\mathbb{C}^2\big(V_{k,1}(t_l)\big),\\
        \mathbb{K}_{k,l} &= \mathbb{K}\big(U_{k,1}(t_l)\big)+\mathbb{K}\big(V_{k,1}(t_l)\big),\\
        \mathbb{B}_{k,l} &= \frac{1}{\mathbb{E}[V_{k,1}(t_l)]} + \frac{2t_l}{\mathbb{E}[U_{k,1}(t_l)]},
    \end{align}
    \noindent where $\mathbb{C}(X)$ and $\mathbb{K}(X)$ is the coefficient of variation and kurtosis of a random variable $X$, respectively \citep{asmussen2008asymptotic}. Then, the regret under the UCB-RB Algorithm is bounded as follows:
    \begin{equation*}
        {\tt REG}_{\pi^{\tt B}}(\tau) \leq \sum_{k}3\alpha \min\{\xi_k,\xi_k^{\prime}\}\log(\tau)+O(KL),
    \end{equation*}
    \noindent where for each $k\in[K]$
    \begin{align}
        \xi_k &= \sum_{l=1}^L\Delta_{k,l}\cdot\mu_{k,l}\cdot \Big\{z^2(\beta)\mathbb{C}_{k,l}\frac{r_*^2}{\Delta_{k,l}^2}+\frac{2z(\beta)r_*\mathbb{B}_{k,l}}{\Delta_{k,l}}+8\big(\mathbb{K}_{k,l}+\mathbb{B}_{k,l}^2\big)\Big\},\\
        \tilde{\xi}_k &= \max_l\{\Delta_{k,l}\cdot\mu_{k,l}\}\cdot \max_l\Big\{z^2(\beta)\mathbb{C}_{k,l}\frac{r_*^2}{\Delta_{k,l}^2}+2z(\beta)\mathbb{B}_{k,l}\frac{r_*}{\Delta_{k,l}}+8\big(\mathbb{K}_{k,l}+\mathbb{B}_{k,l}^2\big)\Big\},
    \end{align}
    \noindent for $\alpha > 2$, $\mu_{k,l} = \mathbb{E}[U_{k,1}(t_l)]$ and
    \begin{equation}
        z(\beta) = \max\Big\{2\sqrt{2}\frac{(1+\beta)^2}{(1-\beta)^3}, \frac{1}{\beta}\Big\},~\beta\in(0,1),\label{eqn:z-beta}
    \end{equation}
    which stems from using the empirical estimates for unknown quantities.
\end{theorem}
\begin{remark}
For any arm $k$, $\xi_k$ corresponds to the coefficient without using the information structure, and grows linearly with $L$, and shows that the dependence of the regret on $\Delta_{k,l}$ is $O(1/\Delta_{k,l})$. On the other hand, $\tilde{\xi}_k$ reflects the effect of exploiting the structure, and it is usually much lower than $C_k$. Hence, we have the following order result for the regret: $${\tt REG}_{\tt RB}(B) = O\Big(t_L^2\big(\sum_{k \neq k^*}\frac{1}{\Delta_k}+\frac{1}{\min_l\Delta_{k^*,l}}+\sum_k\max_l\Delta_{k,l}\big)\log(\tau)\Big),$$ since $\Delta_k \leq \Delta_{k,l}$ for all $l$ for $k\neq k^*$, and $\mathbb{K}(Z)\leq b^2$ for a bounded random variable $Z\in[0,b]$. In other words, as a result of exploiting the information structure, the effect of large $L$ on the regret is eliminated.
\end{remark}

\begin{remark}\label{rem:tl-dependence}
The regret upper bound in Theorem \ref{thm:ucb-rb} grows at a rate $O(t_L^2\log(\tau))$, where the constant additive term is independent of $\tau$ and $t_L$ if $\E[X_{k,1}^p]<\infty$ for some $p>2$. Therefore, if the optimal restart time can take on a large value, then the regret performance deteriorates significantly. This dependence on $t_L$ stems from the nature of the empirical mean estimator used for estimating the reward rate, and it is inevitable \citep{audibert2009exploration}. Therefore, the UCB-RB Algorithm is suitable only for the cases where the restart times are small.
\end{remark}

\begin{proof}
The proof of Theorem \ref{thm:ucb-rb} follows a similar strategy as \citep{cayci2019learning}, and can be found in detail in Appendix \ref{app:ucb-rb}. We will provide a proof sketch here. The main challenge in the proof is two-fold: analyzing the effect of using empirical estimates, and finding tight upper and lower bounds for the expectation of the total reward ${\tt REW}_\pi(\tau)$, which is a controlled and stopped random walk with non-i.i.d. increments. For any $(k, t_l)\in[K]\times\mathbb{T}$, let $T_{k,l}(n)$ be the number of times the controller makes the decision $(k, t_l)$ in the first $n$ stages, and $\Delta_{k,l}$ be the expected 'regret per unit time' if $(k,t_l)$ is chosen. Then, by using tools from renewal theory and martingale concentration inequalities, we express the regret as follows: $${\tt REG}_{\tt \pi^B}(\tau) \leq \sum_{(k,l):\Delta_{k,l}>0}O(1)~\E[T_{k,l}(n_0(\tau))]\E[U_{k,1}(t_l)]\Delta_{k,l} + O(KL),$$ where $n_0(\tau)$ is a high-probability upper bound for $N_{\tt \pi^B}(\tau)$, the total number of pulls in $[0,\tau]$. By using a clean-event bandit analysis akin to \citep{audibert2009exploration} to bound $\E[T_{k,l}(n_0(\tau))]$, we prove the theorem. Note that the UCB-RB Algorithm makes use of the empirical mean and variance estimates to achieve improved regret performance without any prior knowledge, and we devise novel tools to analyze the impact of using empirical estimates on the regret. 
\end{proof}

Remark \ref{rem:tl-dependence} emphasizes a crucial shortcoming of the UCB-RB Algorithm in dealing with large (potentially infinite) waiting times. In order to achieve $O(\log(\tau))$ regret in a more general setting, we design a more general algorithm in the following.

\subsubsection{UCB-RM Algorithm}
If the completion time distributions are such that the optimal restart time is very large or potentially infinite for some arms, the empirical rate estimation fails to provide fast convergence rates, which leads to substantially deteriorated regret bounds. In order to overcome this, we will design a UCB-type policy that incorporates a median-based rate estimator to achieve good performance in a very general setting that allows not restarting as a possible action, i.e., we consider a decision set $$\mathbb{T}=\{t_1<t_2<\ldots<t_L = \infty\},$$ where $t_L=\infty$ implies the controller can wait until the task is completed.

\begin{definition}[Median-based rate estimation]
    Consider $(k,t_l)\in[K]\times\mathbb{T}$, and let $G_1,G_2,\ldots,G_m$ be a partition of $\mathcal{I}_{k,l}^*(n)$ such that $G_j = \lfloor T_{k,l}^*(n)/m \rfloor,~\forall j\in[m]$ for $m=\lfloor 3.5\alpha\log(n)\rfloor + 1$. The median-of-means estimator for $U_{k,n}(t_l)$ is defined as follows:
    \begin{equation}\label{eqn:med-of-means}
        \mathbb{M}_{\mathcal{I}_{k,l}^*(n)}(U_{k}(t_l)) = {\tt median}\Big\{\widehat{\mathbb{E}}_{G_1}[U_{k}(t_l)],\widehat{\mathbb{E}}_{G_2}[U_{k}(t_l)],\ldots,\widehat{\mathbb{E}}_{G_m}[U_{k}(t_l)]\Big\},
    \end{equation}
    \noindent where $\widehat{\mathbb{E}}_S[X]$ is the empirical mean of a sequence $X$ over the index set $S$ defined in \eqref{eqn:emp-mean}. Then, the median-based rate estimator for the action $(k,t_l)$ is defined as follows:
    \begin{equation}\label{eqn:med-rate}
        \widehat{r}^{\tt M}_{k,l,n} = \frac{\mathbb{M}_{\mathcal{I}_{k,l}^*(n)}(V_{k}(t_l))}{\mathbb{M}_{\mathcal{I}_{k,l}^*(n)}(U_{k}(t_l))}.
    \end{equation}
\end{definition}

Similarly, let $$\mathbb{V}^{\tt M}_{\mathcal{I}_{k,l}^*(n)}(U_k(t_l)) = {\tt median}\Big\{\mathbb{V}_{G_1}(U_k(t_l)),\ldots,\mathbb{V}_{G_m}(U_k(t_l))\Big\},$$ where $\mathbb{V}_S$ is the empirical variance defined in \eqref{eqn:emp-var}. By using this median-based variance estimator, let \begin{equation}c_{k,l,n}^{\tt M} = \frac{(1+\beta)^2}{(1-\beta)}\cdot \frac{\eta^{\tt M}_{k,l,n}+\widehat{r}^{\tt M}_{k,l,n}\cdot\epsilon^{\tt M}_{k,l,n}}{\mathbb{M}_{\mathcal{I}_{k,l}^*(n)}(U_k(t_l)},~\beta\in(0,1),\label{eqn:med-radius}\end{equation}\noindent for the confidence radii defined as follows:
\begin{align}
    \epsilon_{k,l,n}^{\tt M} &= 11\sqrt{\frac{2\mathbb{V}^{\tt M}_{\mathcal{I}_{k,l}^*(n)}(U_k(t_l))\log(n^\alpha)}{T_{k,l}^*(n)}},\\
    \eta_{k,l,n}^{\tt M} &= 11\sqrt{\frac{2\mathbb{V}^{\tt M}_{\mathcal{I}_{k,l}^*(n)}(V_k(t_l))\log(n^\alpha)}{T_{k,l}^*(n)}}.
\end{align}
Then, the inequality $\widehat{r}_{k,l,n}^{\tt M}+{c}_{k,l,n}^{\tt M} > r_k(t_l)$ holds with high probability for sufficiently large $T_{k,l}^*(n)$. Based on this construction, the UCB-RM Algorithm makes a decision for the $(n+1)$-th arm pull as follows:
\begin{equation}
    (I_{n+1},\nu_{n+1}) \in \underset{(k,l)\in[K]\times[L]}{\arg\max}~\big\{\widehat{r}_{k,l,n}^{\tt M}+c_{k,l,n}^{\tt M}\big\}.
\end{equation}

In the following theorem, we analyze the performance of the UCB-RM Algorithm.
\begin{theorem}[Regret Upper Bound for {UCB-RM}]\label{thm:ucb-rm}
    The regret under the UCB-RM Algorithm is bounded as follows:
    \begin{equation*}
        {\tt REG}_{\pi^{\tt M}}(\tau) \leq \sum_{k}3\alpha \min\{\xi_k^{\tt M},\tilde{\xi}_k^{\tt M}\}\log(\tau)+O(KL),
    \end{equation*}
    \noindent where for each $k\in[K]$
    \begin{align}
        \xi_k^{\tt M} &=  \sum_{l=1}^L\Big\{11^2z^2(\beta)\mathbb{C}_{k,l}\frac{r_*^2}{\Delta_{k,l}}+1024\big(\mathbb{K}_{k,l}+\zeta\big)\Delta_{k,l}\Big\},\\
        \tilde{\xi}_k^{\tt M} &= \max_l\Delta_{k,l}\cdot \max_l\Big\{11^2z^2(\beta)\mathbb{C}_{k,l}\frac{r_*^2}{\Delta_{k,l}^2}+1024\big(\mathbb{K}_{k,l}+\zeta\big)\Big\},
    \end{align}
    \noindent for some constant $\zeta > 0$, $\alpha > 2$ and $z(\beta)$ defined in \eqref{eqn:z-beta}.
\end{theorem}
\begin{remark}\normalfont
Notice that the regret upper bound in Theorem \ref{thm:ucb-rm} is upper bounded by functions that depend only the first- and second-order moments of $X_{k,n}$ and $R_{k,n}$, independent of $t_L$. Therefore, the UCB-RM Algorithm is efficient for the cases with very large (potentially infinite) $t_L$ unlike UCB-RB, whose regret grows at a rate $O(t_L^2\log(\tau))$. This generality comes at a price: comparing the coefficients of the $\log(\tau)$ term in Theorem \ref{thm:ucb-rb} and Theorem \ref{thm:ucb-rm}, we observe that the UCB-RM Algorithm suffers from a considerably large scaling coefficient. This suggests that if the optimal restart times are known to be small, then the UCB-RB Algorithm is more efficient than the UCB-RM Algorithm.
\end{remark}

Also, note that the UCB-RM Algorithm does not require any prior knowledge unlike the median-based algorithm {\tt UCB-BwI} in \citep{cayci2019learning}. Instead, it uses empirical estimates for reward rate, mean completion time and variances. The regret upper bound for UCB-RM is tighter compared to {\tt UCB-BwI}.

In the next section, we will develop a learning algorithm for the case $\mathbb{T}$ is continuous.

\subsection{Continuous Set of Restart Times: UCB-RC}
For a broad class of completion time and reward distributions, the reward rate function $r_k(t)$ has a smooth and unimodal structure. By exploiting this property, we can design learning algorithms for a continuous set of restart times $\mathbb{T}$ achieving sublinear regret. In the following, we will propose a learning algorithm for smooth and unimodal $r_k(t)$ based on the UCB-RB Algorithm and the bandit optimization methodology in \citep{combes2014unimodal}. Note that the information structure discussed in Section \ref{subsec:info-structure}  

For the sake of simplicity in exposition, we will consider learning the optimal restart time for a single arm. The extension to $K>1$ is straightforward.

\begin{assumption}
We make the following assumptions on the action set $\mathbb{T}$ and reward rate function $r_1(t)$.
\begin{enumerate}[label=(\roman*)]
    \item Compactness: The decision set $\mathbb{T}$ is a compact subset of $\mathbb{R}_+$:
    \begin{equation}
        \mathbb{T}=[t_{min},t_{max}],
        \label{eqn:compactness}
    \end{equation}
    where $0<t_{min}\leq t_{max}<\infty$. Furthermore, $\mathbb{T}$ satisfies Assumption \ref{assn:action-set} for some $\epsilon,\mu_*>0$ for efficient estimation of $r_1(t)$.
    \item Unimodality: There is an optimal restart time $t_{1}^*\in[t_{min},t_{max}]$ such that $$r_* = r_1(t_{1}^*) \geq r_1(t),$$ for all $t\in\mathbb{T}\backslash\{t_{1}^*\}$.
    \item Smoothness: There exists $\delta_0>0, q > 1$ such that:
        \begin{itemize}
            \item For all $t, t^\prime\in [t_{1}^*,t_{1}^*+\delta_0]$ (or $[t_{1}^*-\delta_0, t_{1}^*]$), the following holds: $$a_1|t-t^\prime|^q \leq |r_1(t)-r_1(t^\prime)|,~a_1 >0,$$
            \item For some $\delta \leq \delta_0$, if $|t-t_{1}^*|\leq \delta$, we have: $$r_*-r_1(t) \leq a_2\delta^q.$$
        \end{itemize}
\end{enumerate}
\label{assn:ucb-rc}
\end{assumption}
\noindent Note that Assumption \ref{assn:ucb-rc} is satisfied for a broad class of distributions. For example, if $X_{k,n}$ and $R_{k,n}$ are independent and $X_{k,n}$ has a uniform, exponential or Pareto distribution, then the conditions are trivially satisfied.

The UCB-RC Algorithm is defined as follows:
\begin{enumerate}
    \item For $\delta = \Big(\sqrt{\log(\tau)/\tau}\Big)^{1/q}$ and ${\tt rad}(\mathbb{T}) = t_{max}-t_{min}$, let $L(\delta) = {\tt rad(\mathbb{T})}/{\delta}$ and $\mathbb{T}_Q = \{t_1,t_2,\ldots,t_{L(\delta)}\}$ where $$t_l = t_{min}+(l-1)\cdot\lceil 1/\delta\rceil,~l=1,2,\ldots,L(\delta).$$
    \item Run the UCB-RB Algorithm over the action set $\mathbb{T}_Q$.
\end{enumerate}

The following theorem provides a regret bound for the UCB-RC Algorithm.

\begin{theorem}[Regret Upper Bound for UCB-RC]
    Under Assumption \ref{assn:ucb-rc}, the regret under the UCB-RC Algorithm satisfies the following asymptotic upper bound:
    \begin{equation}
        \lim\sup_{\tau\rightarrow\infty}~\frac{{\tt REG}_{\tt \pi^C}(\tau)}{\sqrt{\tau\log(\tau)}} \leq a_26^q\frac{\E[X_{1,1}]}{\mu_*} + \frac{3\alpha q}{a_1(q-1)}\mathbb{C}^\star z^2(\beta)r_*^2,
    \end{equation}
    where $$\mathbb{C}^\star = \Big(\frac{\E[X_{1,1}^2]}{\mu_*^2}+\frac{\E[R_{1,1}^2]}{\epsilon^2\big(\E[R_{1,1}]\big)^2}\Big).$$
    \label{thm:ucb-rc}
\end{theorem}
\begin{proof}
The detailed proof of Theorem \ref{thm:ucb-rc} can be found in Appendix \ref{app:ucb-rc}. Note that the UCB-RC Algorithm is based on quantizing the decision set $\mathbb{T}$, and running the UCB-RB Algorithm over the quantized decision set $\mathbb{T}_Q$. In the proof, we show that the step size $\delta$ is chosen such that the optimal reward rate over $\mathbb{T}_Q$ is close enough to $r_*$ while the number of quantization levels are kept sufficiently small at the same time. Under the compactness and smoothness assumptions summarized in Assumption \ref{assn:ucb-rc}, the regret upper bound is obtained.
\end{proof}

The UCB-RC Algorithm is based on an extension of the UCB-type algorithm for unimodal discrete-time stochastic bandits proposed in \citep{combes2015bandits}. A straightforward extension of the algorithm in \citep{combes2015bandits} would yield $O(\log(\tau)\sqrt{\tau})$ regret. However, the UCB-RC Algorithm achieves $O(\sqrt{\tau\log(\tau))}$ regret in this case. The order reduction by a factor of $O(\sqrt{\log(\tau)})$ is because UCB-RC incorporates the information structure that stems from the right-censored feedback, discussed in Section \ref{subsec:info-structure}.


\section{Numerical Experiments}
In this section, we evaluate the performance of the proposed learning algorithm for boosting the WalkSAT Algorithm on Random-SAT benchmark sets. The experiments are conducted in a similar manner as the SAT Competition: for a given time interval $[0,\tau]$, the performance metric is the number of solved problem instances, thus there is a unit reward for each successful assignment, i.e., $R_{1, n}=1$.

\textbf{SAT Solver:} We used the C implementation of the WalkSAT Algorithm provided in \citep{kautz-imp} with the heuristics {\tt Rnovelty}.

\textbf{Methodology:} For a fair and universal comparison, we measure the completion time of a problem by the number of flips performed by the WalkSAT Algorithm. We allowed at most 10 restarts for each problem instance. As the number of benchmark instances is small, we generated i.i.d. random samples from the empirical distribution of the completion times by using the inverse transform method whenever the number of samples exceeds the dataset length \citep{ross2014introduction}.

\subsection{Uniform Random-3-SAT}
In the first example, we will evaluate the performance of the {UCB-RB} Algorithm on a Random-3-SAT dataset.

\textbf{Dataset Description:} We evaluated the performance of the meta-algorithms over the widely used Uniform Random-3-SAT benchmark set of satisfiable problem instances in the SATLIB library \citep{hoos2000satlib}. In the dataset {\tt uf-100-430}, there are 1000 uniformly generated problem instances with 100 variables and 430 clauses, therefore it is reasonable to assume i.i.d. completion times. Each successful assignment yields a reward $R_{1,n}=1$.

\textbf{Completion Time Statistics:} The empirical reward rate as a function of the restart time for the data set is given in Figure \ref{fig:uf100}(a). It is observed that the controlled restarts are essential for optimal performance, in accordance with the power-law completion time distributions \citep{gomes1998boosting} and Theorem \ref{thm:opt-int}. This implies that our design is suitable for this scenario.

\textbf{Performance Results:} In this set of experiments, we used the UCB-RB Algorithm with $\alpha = 2.01$, $(1+\beta)^2/(1-\beta)=1.01$ and $\mathbb{T}=\{10^{-0.5+i\times 0.125}:i = 0,1,\ldots,8\}$. For initialization, the controller performed 40 trials for each $(k,t_l)$ decision. For comparison, we used Luby restart strategy with various hand-tuned base cutoff values as a benchmark \citep{luby1993optimal}. Note that without any prior knowledge, the performance of Luby restart strategy is hit-or-miss, depending on how close the chosen (guessed) base cutoff value is to the optimal restart time. The number of solved problem instances for different $\tau$ values are given in Figure \ref{fig:uf100}. We observe that the {UCB-RB} Algorithm learns the optimal restart strategy fast without any prior information, and its performance outperforms alternatives especially at large time horizons. On the other hand, Luby restart strategy, which requires the base cutoff value as an input, is prone to perform badly with inaccurate prior information. Even a genie provides a well-chosen base cutoff value to Luby restart strategy, it is outperformed linearly over $\tau$ by the UCB-RB Algorithm, which requires no prior information. 

\begin{figure}
    \centering
    \includegraphics[scale=0.5]{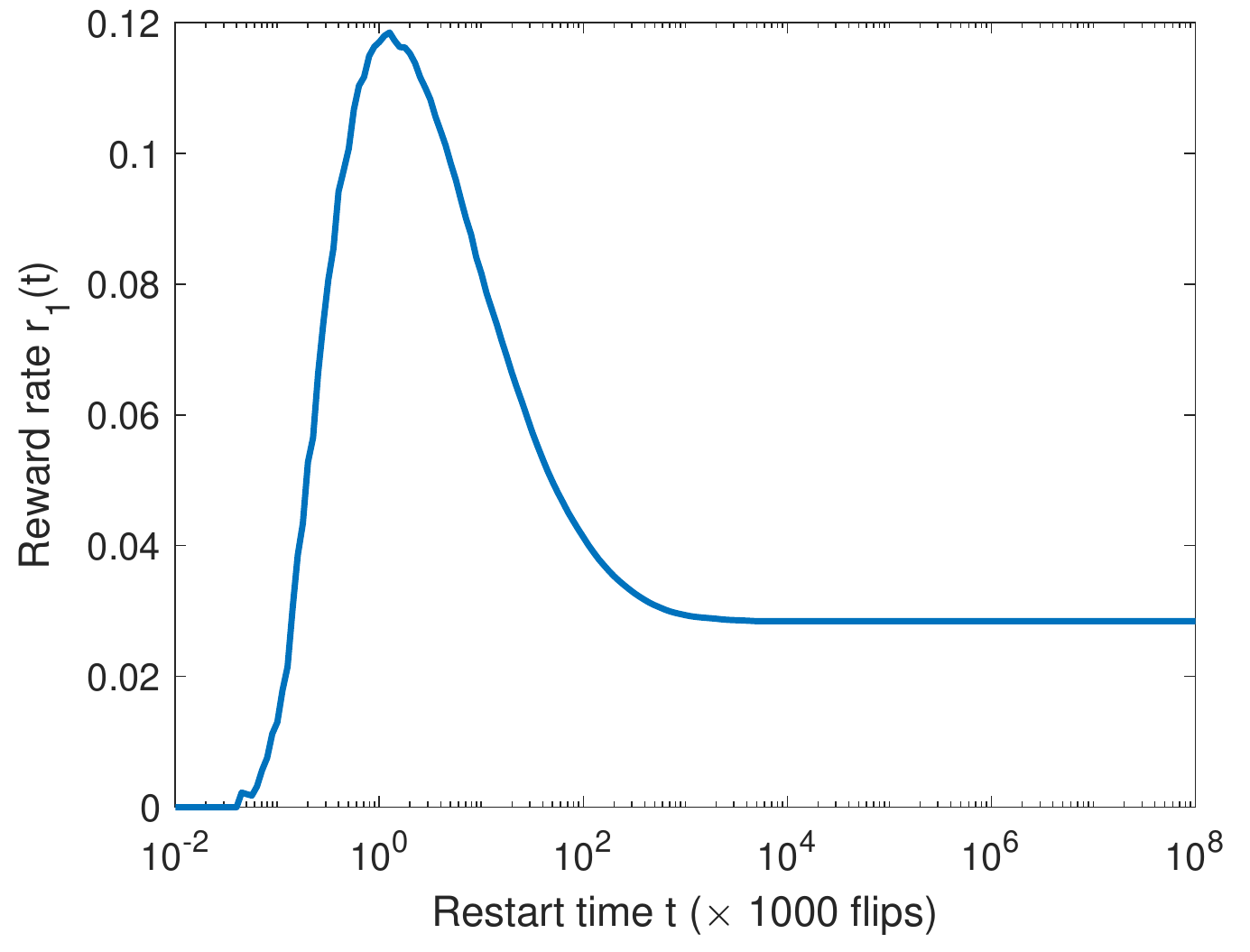}
    \includegraphics[scale=0.5]{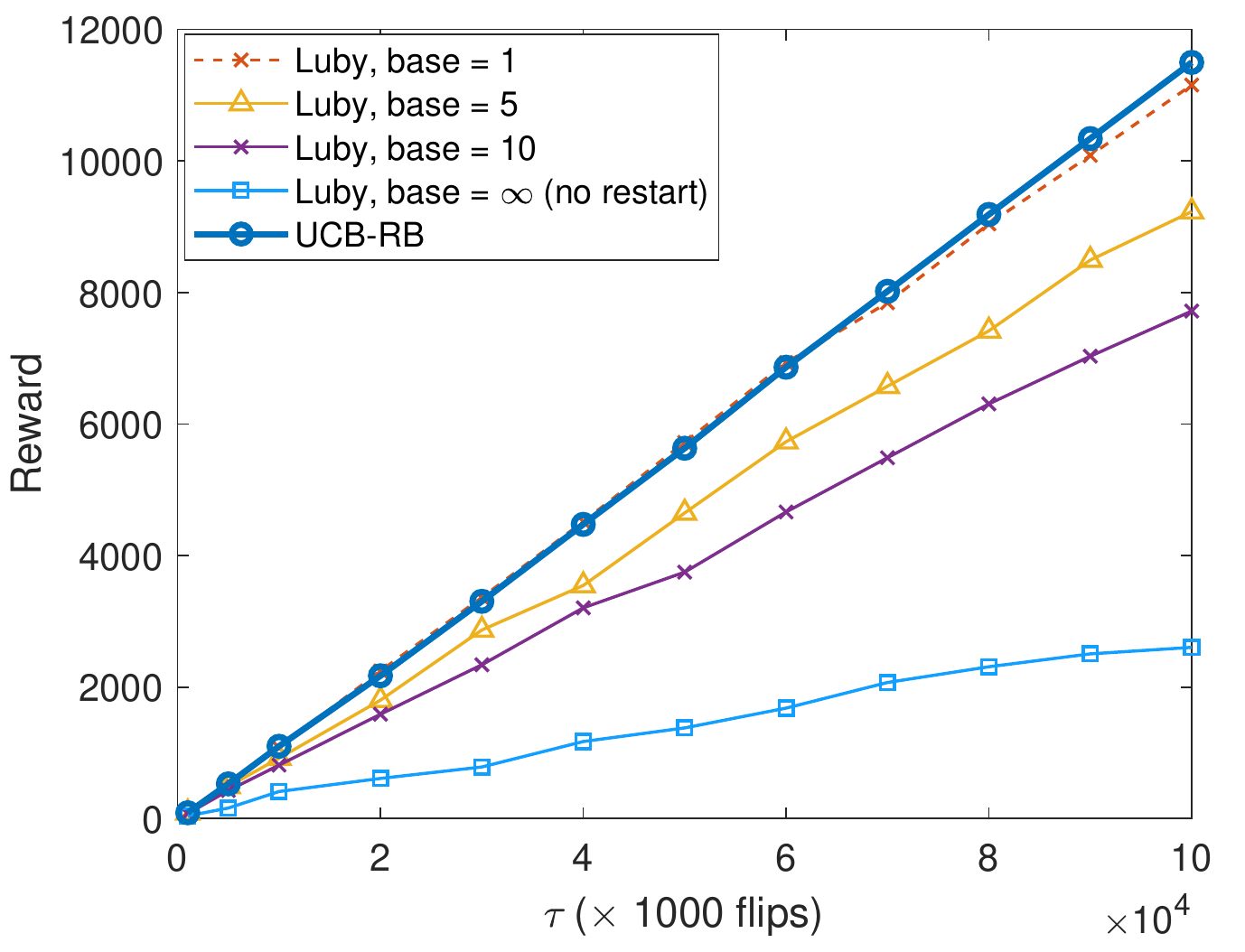}
    \caption{Performance of the restart strategies on the Random-3-SAT data set {\tt cbs-k3}.}
    \label{fig:uf100}
\end{figure}

\subsection{Random-3-SAT Instances with Controlled Backbone Size}
In this example, we will evaluate the performance of the UCB-RB Algorithm on a Random-3-SAT dataset with controlled backbone size. 

\textbf{Dataset Description:} We evaluated the performance of the meta-algorithms over the widely used Random-3-SAT benchmark set of satisfiable problem instances in \citep{hoos2000satlib}. In the dataset {\tt CBS\_k3\_n100\_m403\_b30}, there are 1000 uniformly random generated problem instances with 100 variables and 403 clauses with backbone size 30, therefore it is reasonable to assume i.i.d. completion times.

\textbf{Completion Time Statistics:} The empirical reward rate as a function of the restart time for the data set is given in Figure \ref{fig:cbs-k3}(a).

\textbf{Performance Results:}  In this set of experiments, we used the UCB-RB Algorithm with $\alpha = 2.01$, $(1+\beta)^2/(1-\beta)=1.01$ and $\mathbb{T}=\{10^{-0.5+i\times 0.125}:i = 0,1,\ldots,12\}$. Similar to the previous example, we used Luby restart strategy with various hand-tuned base cutoff values as a benchmark. The number of solved problem instances for different $\tau$ values are given in Figure \ref{fig:cbs-k3}. 
\begin{figure}
    \centering
    \includegraphics[scale=0.5]{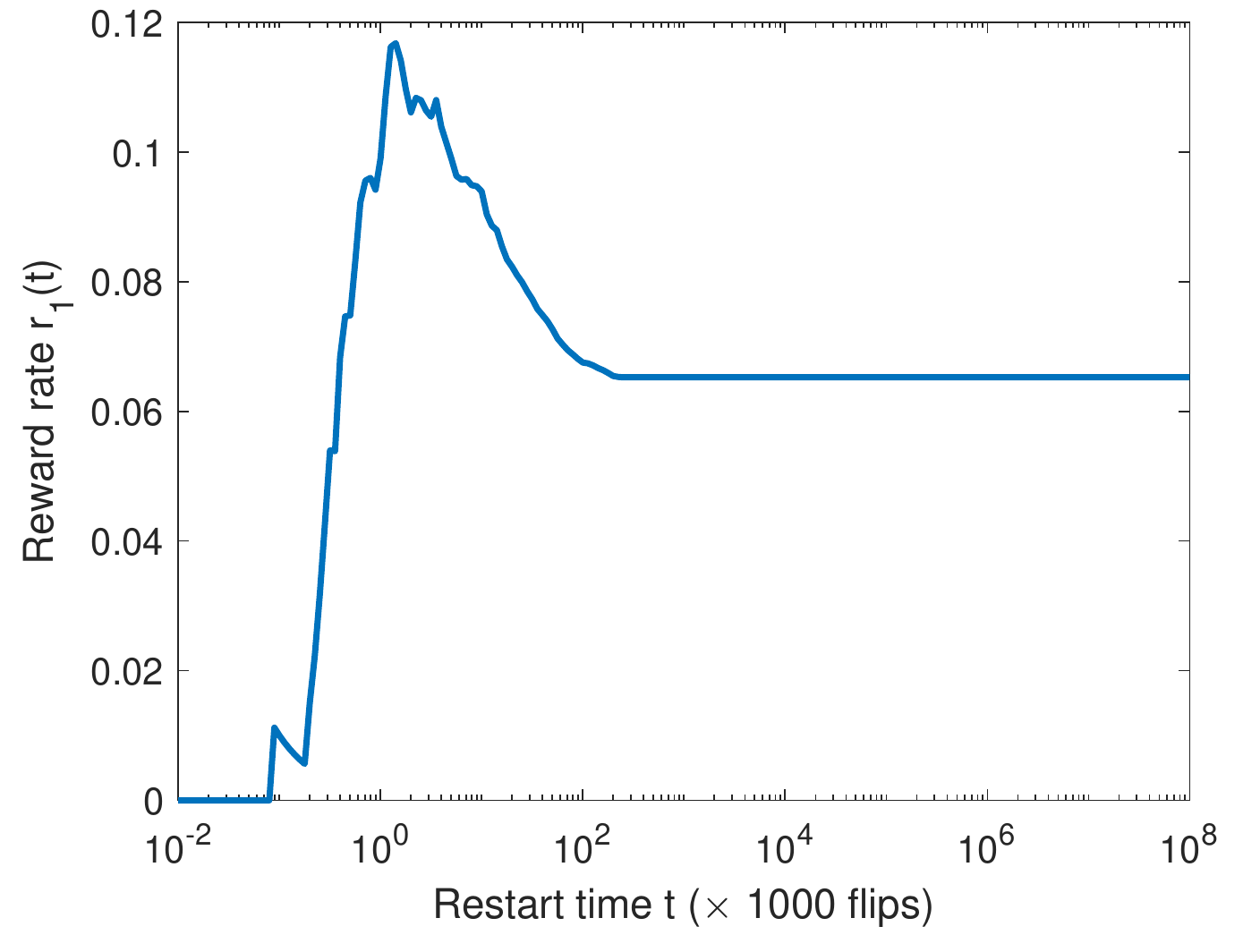}
    \includegraphics[scale=0.5]{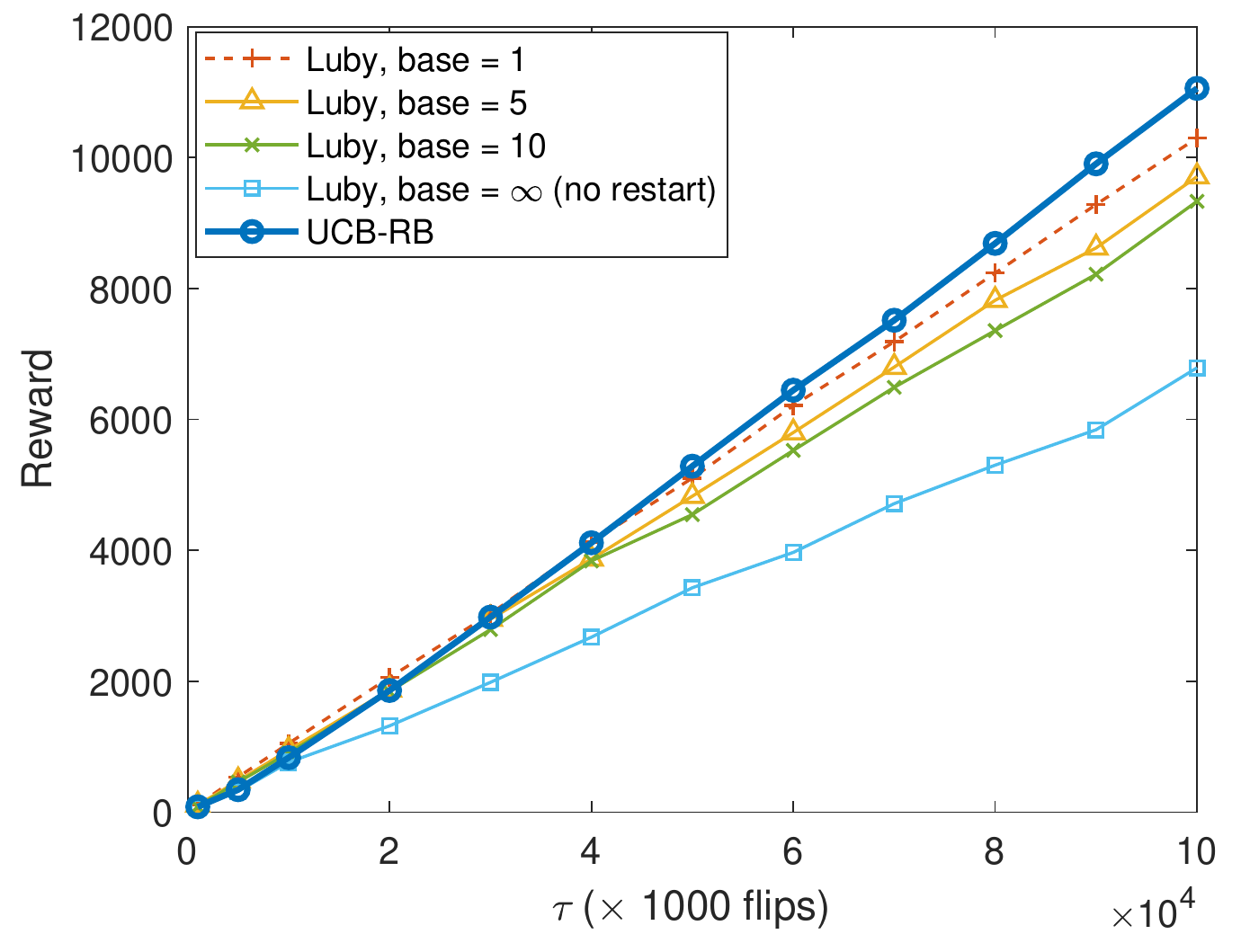}
    \caption{(Left) Empirical reward rate as a function of restart time for the dataset {\tt CBS\_k3\_n100\_m403\_b30}. Restart strategies lead to substantial gains on the number of problems solved per unit time. (Right) Performance of the UCB-RB Algorithm and Luby strategy for various base cutoff values.}
    \label{fig:cbs-k3}
\end{figure}
\noindent Figure \ref{fig:cbs-k3} indicates that the {UCB-RB} Algorithm learns the optimal restart strategy with low regret without any prior information. The performance gap between the UCB-RB Algorithm and Luby strategy increases linearly over $\tau$.

\section{Conclusions}
In this paper, we considered the continuous-time bandit learning problem with controlled restarts, and presented a principled approach with rigorous performance guarantees. For correlated and potentially heavy-tailed completion time and reward distributions, we proposed a simple, intuitive and near-optimal offline policy with $O(1)$ optimality gap, and characterized the nature of optimal restart strategies by using this approximation. For online learning, we considered discrete and continuous action sets, and proposed bandit algorithms that exploit the statistical structure of the problem to achieve tight performance guarantees. In addition to the theoretical analysis, we evaluated the numerical performance to boost the speed of SAT solvers in random 3-SAT instances, and observed that the learning solution proposed in this paper outperforms Luby restart strategy with no prior information.



\newpage

\appendix
\section{Concentration Inequalities for Reward Rate Estimation}
In this section, we will provide tight concentration inequalities that employ empirical estimates. The following lemma will provide a basis to design these concentration inequalities.

\begin{lemma}[Proposition 2, \citep{cayci2020budget}]\label{lem:rate-conc}
Consider a pair of parameters $\mu_X,\mu_R > 0$, and their estimators $\widehat{\mu}_X$ and $\widehat{\mu}_R$, respectively. Let $\widehat{r}= \widehat{\mu}_R/\widehat{\mu}_X$ and $r = \mu_R/\mu_X$. Then, for any $\beta \in (0,1)$, we have the following inequality:
\begin{equation}
    \p\Big(|\widehat{r}-r| > \frac{(1+\beta)^2}{1-\beta}\frac{\eta+\widehat{r}\cdot\epsilon}{\widehat{\mu}_X}\Big) \leq \p(|\mu_X-\widehat{\mu}_X|>\epsilon)+\p(|\mu_R-\widehat{\mu}_R| > \eta),
\end{equation}
\noindent for any $\eta \leq \beta\mu_R$,and $\epsilon \leq \beta\mu_X$.
\end{lemma}

\begin{proof}
    Let $$A(\epsilon,\eta) = \{|\mu_X-\widehat{\mu}_X|\leq \epsilon\}\cap\{|\mu_R-\widehat{\mu}_R|\leq \eta\},$$ be the high-probability event. Then, by Proposition 1 in \citep{cayci2020budget}, we have the following set relation: $$A(\epsilon,\eta) \subset \Big\{ |\widehat{r}-r| \leq \frac{\eta+r\epsilon}{\mu_X-\epsilon} \Big\}.$$ For any $\beta \in (0,1)$, if $\epsilon\leq\beta\mu_X$ and $\eta\leq\beta\mu_R$ is satisfied, then we have: $$\frac{\eta+r\epsilon}{\mu_X-\epsilon} \leq\frac{(1+\beta)^2}{1-\beta} \frac{\eta+\widehat{r}\epsilon}{\widehat{\mu}_X},$$ within the set $A(\epsilon,\eta)$. By taking the compliment and using union bound, we obtain the result. 
\end{proof}

By using Lemma \ref{lem:rate-conc}, we can prove tight concentration bounds for the renewal rate as follows.

\begin{proposition}[Concentration inequalities for reward rate]\label{prop:concentration}
Let $\{(X_n,R_n):n\geq 1\}$ be a renewal reward process.
\begin{enumerate}
    \item If $X_n\in[0,a]$ and $R_n\in [0,b]$, let $$\widehat{r}_n = \frac{\widehat{\E}_{[n]}[X]}{\widehat{\E}_{[n]}[R]},$$ and
    \begin{align*}
        \epsilon_n(\delta) &= \sqrt{\frac{2\mathbb{V}_{[n]}(X)\log(1/\delta)}{n}}+\frac{3a\log(1/\delta)}{n},\\
         \eta_n(\delta) &= \sqrt{\frac{2\mathbb{V}_{[n]}(R)\log(1/\delta)}{n}}+\frac{3b\log(1/\delta)}{n}.
    \end{align*}
    Then, for any $\beta \in (0,1)$ and $\delta \in (0,1)$, we have: $$\p\Big(|\widehat{r}_n-r| > \frac{(1+\beta)^2}{1-\beta}\cdot \frac{\eta_n(\delta)+\widehat{r}_n\epsilon_n(\delta)}{\widehat{\E}_{[n]}[X]}\Big) \leq 12\delta,$$ for any $n\geq 8\Big(\mathbb{K}(X_1)+\mathbb{K}(R_1)+\frac{a^2}{Var(X_1)}+\frac{b^2}{Var(R_1)}\Big) + 3\Big(\frac{\mathbb{C}^2(X_1)}{\beta^2} + \frac{a}{\beta\E[X_1]} + \frac{\mathbb{C}^2(R_1)}{\beta^2} + \frac{b}{\beta\E[R_1]}\Big)$.
    
\item Consider a renewal reward process such that $\E[X_1^4]<\infty$ and $\E[R_1^4]<\infty$. For $m = \lfloor 3.5\log(1/\delta)\rfloor + 1$, let $G_1,G_2,\ldots, G_m$ be a partition of $[n]$ such that $|G_j|=\lfloor n/m\rfloor$. Then, we define the median-based mean and variance estimators as follows:
\begin{align*}
    \mathbb{M}_{[n]}(X) &= {\tt median}\{ \widehat{\E}_{G_1}[X],\widehat{\E}_{G_2}[X],\ldots,\widehat{\E}_{G_m}[X] \},\\
    \mathbb{V}^{\tt M}_{[n]}(X) &= {\tt median}\{ \mathbb{V}_{G_1}[X],\mathbb{V}_{G_2}[X],\ldots,\mathbb{V}_{G_m}[X] \}.
\end{align*}
Let $$\widehat{r}_{n}^{\tt M} = \frac{\widehat{\E}_{[n]}^{\tt M}[R]}{\widehat{\E}_{[n]}^{\tt M}[X]},$$ and 
    \begin{align*}
        \epsilon_n^{\tt M}(\delta) &= 11\sqrt{\frac{2\mathbb{V}^{\tt M}_{[n]}(X)\log(1/\delta)}{n}},\\
         \eta_n^{\tt M}(\delta) &= 11\sqrt{\frac{2\mathbb{V}^{\tt M}_{[n]}(R)\log(1/\delta)}{n}}.
    \end{align*}
    Then, for any $\beta \in (0,1)$ and $\delta \in (0,1)$, we have: $$\p\Big(|\widehat{r}_n^{\tt M}-r| > \frac{(1+\beta)^2}{1-\beta}\cdot \frac{\eta_n^{\tt M}(\delta)+\widehat{r}_n^{\tt M}\epsilon_n^{\tt M}(\delta)}{\widehat{\E}^{\tt M}_{[n]}[X]}\Big) \leq 16.8\delta,$$ for $n \geq 1024\big(\mathbb{K}(X_1)+\mathbb{K}(R_1)+\frac{\mathbb{C}^2(X_1)+\mathbb{C}^2(R_1)}{\beta^2}+\zeta\big)$ for some $\zeta > 0$.
\end{enumerate}
\end{proposition}

\begin{proof}
\begin{enumerate}
    \item If $n \geq 8\big(\mathbb{K}(X_1)+\frac{a^2}{Var(X_1)}\big)$, then the following holds with probability at least $1-4\delta$: $$|\mathbb{V}_{[n]}(X_1)-Var(X_1)| \leq Var(X_1)/2,$$ where $\mathbb{K}(X)=\frac{\E|X-\E X|^2}{Var^2(X)}$ is the kurtosis of a random variable $X$. Therefore, with probability at least $1-4\delta$, we have the following: 
    \begin{align*}
    \sqrt{\frac{3Var(X_1)\log(1/\delta)}{n}}+\frac{3a\log(1/\delta)}{n} &\leq \beta\E[X_1].
 \end{align*}
 Hence, if $n \geq 3\Big(\frac{\mathbb{C}^2(X_1)}{\beta^2} + \frac{a}{\beta\E[X_1]}\Big)$ also holds, then the above inequality is automatically satisfied, where $\mathbb{C}(X)=\sqrt{Var(X)}/\E[X]$ is the coefficient of variation. Therefore, if $$n\geq 8\Big(\mathbb{K}(X_1)+\mathbb{K}(R_1)+\frac{a^2}{Var(X_1)}+\frac{b^2}{Var(R_1)}\Big) + 3\Big(\frac{\mathbb{C}^2(X_1)}{\beta^2} + \frac{a}{\beta\E[X_1]} + \frac{\mathbb{C}^2(R_1)}{\beta^2} + \frac{b}{\beta\E[R_1]}\Big),$$ then we have
 \begin{align*}
     \epsilon_n(\delta) &\leq \beta\E[X_1],\\
     \eta_n(\delta) &\leq \beta\E[R_1],
 \end{align*}
 with probability at least $1-8\delta$. Then we use Lemma \ref{lem:rate-conc} in conjunction with empirical Bernstein inequality (see \citep{audibert2009exploration}) to conclude the proof.
 \item The proof follows from identical steps as Part 1, and uses Proposition 4.1 and Corollary 4.2 in \citep{minsker2015geometric} for the concentration results.
\end{enumerate}
\end{proof}

\section{Proof of Theorem \ref{thm:ucb-rb}}\label{app:ucb-rb}
The number of trials $N_\pi(\tau)$ under an admissible policy $\pi$ is a random stopping time, which makes the regret computations difficult. The following proposition provides a useful tool for regret computations.

    \begin{lemma}[Regret Upper Bounds for Admissible Policies]\label{lem:reg-dec}
Let $T_{k,l}(n)$ be the number of steps where the decision is $(k,t_l)$ in $n$ trials, and $\mu_* = \min_{k,t}\E[U_{k,1}(t_l)]$. The following upper bound holds for any admissible policy $\pi\in\Pi$ and $\tau > \mu_*/2$:
\begin{equation*}\label{eqn:reg-dec}
    {\tt REG}_\pi(\tau) \leq \sum_{k,l} \E\Big[T_{k,l}\Big(\frac{2\tau}{\mu_*}\Big)\Big]\Delta_{k,l}\E[U_{k,1}(t_l)] + \frac{\exp(-\tau\mu_*/t_1^2)}{1-\exp\big(\mu_*^2/(2t_1^2)\big)}\sum_{k,l} \Delta_{k,l}\E[U_{k,1}(t_l)]+r^*\Phi,
\end{equation*}
\noindent where $\Phi$ is a constant.
\end{lemma}

The proof of Lemma \ref{lem:reg-dec} relies on Azuma-Hoeffding inequality for controlled random walks, which can be found in \citep{cayci2019learning, cayci2020budget}. Note that $2\tau/\mu_*$ is a high-probability upper bound for the total number of pulls $N_\pi(\tau)$, and $\Delta_{k,l}\E[U_{k,1}(t_l)]$ is the average regret per pull for a decision $(k,t_l)$. Lemma \ref{lem:reg-dec} implies that the expected regret after $2\tau/\mu^*$ pulls is $O(1)$.

In the following lemma, we quantify the scaling effect of using empirical estimates.

\begin{lemma}
Under the UCB-RB Algorithm, we have the following upper bounds:
\begin{enumerate}[label=(\roman*)]
    \item If $k=k^*,t_l>t_{k^*}^*$ or $k\neq k^*,\forall l$, we have:
    \begin{equation*}
    \sum_{j=l}^L\mathbb{E}[T_{k,j}(N)]\leq 3\Big(z^2(\beta)\mathbb{C}_{k,j}\frac{r_*^2}{\Delta_{k,j}^2}+2z(\beta)\mathbb{B}_{k,j}\frac{r_*}{\Delta_{k,j}}+8\big(\mathbb{K}_{k,j}+\mathbb{B}_{k,j}^2\big)\Big)\log(N^\alpha)+O(L),
\end{equation*}
\item If $k=k^*$ and $t_l < t_{k^*}^*$, we have $\mathbb{E}[T_{k,l}(n)] = O(1)$ for all $n$.
\end{enumerate}
\label{lem:sample-complexity}
\end{lemma}
\begin{proof}
Consider a suboptimal decision $(k,t_l)$ where either $\{k=k^*, t_l > t_k^*\}$ or $\{k\neq k^*\}$ is true, and let
\begin{align}
    E_{1,n} &= \Big\{\widehat{r}_*+c_* \leq r_*\Big\},\\
    E_{2,n} &= \bigcup_{j\geq l}\Big\{\widehat{r}_{k,j,n}\geq c_{k,j,n} + r_k(t_j)\Big\},\\
    E_{3,n} &= \bigcup_{j\geq l}\{2c_{k,j,n} \geq \Delta_{k,j}\}.
\end{align}
Then, it is trivial to show by using contradiction that $\{(I_{n+1},\nu_{n+1}) = (k,t_j):j\geq l\} \subset \cup_{i=1}^3E_{i,n}$. In the following, we provide a sample complexity analysis for the events above.

For notational simplicity, for $j\geq l$, let $U_n:=U_{k,n}(t_j)$ and $V_n:=V_{k,n}(t_j)$ for all $n$. For sample size $s$ and any $\delta \in(0,1)$, let $\widehat{r}_s = \frac{\widehat{\E}_{[s]}[V]}{\widehat{\E}_{[s]}[U]}$, and
\begin{align}
\epsilon_s &= \sqrt{\frac{2\mathbb{V}_{[s]}(U)\log(1/\delta)}{s}}+\frac{3t_j\log(1/\delta)}{s},\\
\eta_s &= \sqrt{\frac{2\mathbb{V}_{[s]}(V)\log(1/\delta)}{s}}+\frac{3\log(1/\delta)}{s}.
\end{align}
\noindent Then, by Proposition \ref{prop:concentration}, the following holds with probability at least $1-4\delta$: \begin{equation}\Big|\widehat{r}_s-r_k(t_j)\Big| \leq \frac{(1+\beta)^2}{1-\beta}\frac{\eta_s+\widehat{r}_s\epsilon_s}{{\widehat{\E}_{[s]}[U]}},
\label{eqn:concentration-0}\end{equation} given $s$ is sufficiently large such that 
\begin{align}
    \begin{aligned}
        \epsilon_s &\leq \beta\E[U_1], \\
        \eta_s&\leq \beta\E[V_1].
    \end{aligned}
    \label{eqn:stability-condition}
\end{align}
By using empirical Bernstein inequality and union bound for the empirical variance, we have the following inequalities with probability at least $1-8\delta$: 
\begin{align*}
    |\mathbb{V}_{[s]}(U)-Var(U_1)|\leq Var(U_1)/2,\\
    |\mathbb{V}_{[s]}(V)-Var(V_1)|\leq Var(V_1)/2,
\end{align*}if $s \geq s_{1,j}(\delta) := 8\big(\mathbb{K}(U_1)+\mathbb{K}(V_1)+\frac{4t_j^2}{Var(U_1)}+\frac{1}{Var(V_1)}\Big)\log(1/\delta)$. Therefore, the condition \eqref{eqn:stability-condition} implies the following with probability at least $1-8\delta$ for $s \geq s_{1,j}(\delta)$:
\begin{align}
    \begin{aligned}
    \tilde{\epsilon}_s&:=\sqrt{\frac{3Var(U_1)\log(1/\delta)}{s}}+\frac{3t_j\log(1/\delta)}{s} \leq \beta\E[U_1],\\
 \tilde{\eta}_s&:=\sqrt{\frac{3Var(V_1)\log(1/\delta)}{s}}+\frac{3\log(1/\delta)}{s} \leq \beta\E[V_1].
 \end{aligned}
 \label{eqn:stability-condition-2}
\end{align}
The inequalities in \eqref{eqn:stability-condition-2} simultaneously hold if $$s \geq s_{2,j}(\delta) := 3\Big(\frac{1}{\beta^2}\big(\mathbb{C}^2(U_1)+\mathbb{C}^2(V_1)\big) + \frac{1}{\beta}\big(\frac{t_j}{\E[U_1]}+\frac{1}{\E[V_1]}\big)\Big)\log(1/\delta).$$ Also, note that $$\frac{(1+\beta)^2}{1-\beta}\frac{\eta_s+\widehat{r}_s\epsilon_s}{{\widehat{\E}_{[s]}[U]}} \leq \frac{(1+\beta)^2}{(1-\beta)^3}\frac{\tilde{\eta}_s+r_k(t_j)\tilde{\epsilon}_s}{{\E[U]_1}},$$ holds if \eqref{eqn:stability-condition-2} is true. Thus, by using this result, we show that $$2\frac{(1+\beta)^2}{1-\beta}\frac{\eta_s+\widehat{r}_s\epsilon_s}{{\widehat{\E}_{[s]}[U]}} \leq 2\frac{(1+\beta)^2}{(1-\beta)^3}\frac{\tilde{\eta}_s+r_k(t_j)\tilde{\epsilon}_s}{{\E[U]_1}} \leq \Delta_{k,j},$$ with probability at least $1-8\delta$ if $s \geq \max\{s_{1,j}(\delta),s_{2,j}(\delta), s_{3,j}(\delta)\}$ for $$s_{3,j}(\delta) = 3\Big(8\frac{(1+\beta)^4}{(1-\beta)^6}\frac{r_*^2}{\Delta_{k,j}^2}(\mathbb{C}^2(U_1)+\mathbb{C}^2(V_1))+2\frac{(1+\beta)^2}{(1-\beta)^3}(\frac{1}{\E[V_1]}+\frac{2t_j}{\E[U_1]})\Big)\log(1/\delta).$$ In summary, the following events simultaneously hold with probability at least $1-12\delta$:
\begin{align}
    \begin{aligned}
    \Big|\widehat{r}_s-r_k(t_j)\Big| &\leq \frac{(1+\beta)^2}{1-\beta}\frac{\eta_s+\widehat{r}_s\epsilon_s}{{\widehat{\E}_{[s]}[U]}}\leq \Delta_{k,j}/2,
    \end{aligned}
    \label{eqn:concentration}
\end{align}
if $s \geq \max\{s_{1,j}(\delta), s_{2,j}(\delta), s_{3,j}(\delta)\}$.

Note that the UCB-RB Algorithm is designed such that $s=T_{k,j}^*(n)$ and $\delta = n^{-\alpha}$. Now we will use the above analysis to provide an upper bound for $\sum_{j\geq l}\E[T_{k,j}(n)]$. First, let $$A_j = \{T_{k,j}^*(n) \leq u_j\},~j\geq l$$ for $u_j=\max_i\{s_{i,j}(N^{-\alpha})\}$. Then, by \eqref{eqn:concentration-0} and \eqref{eqn:concentration}, we have the following inequality:
\begin{equation}
    \p\Big(\big(\cup_{j\geq l}A_j\big)\cap\big(E_{1,n}\cup E_{2,n}\cup E_{3,n}\big)\Big) \leq 16(L-l)/n^{\alpha-1},
    \label{eqn:decomposition}
\end{equation}
where we used union bound to deal with the random sample size $T_{k,j}*(n)\leq n$ in computing probabilities. Now, recall that $T_{k,j}^*(n) = \sum_{j^\prime\geq j}T_{k,j^\prime}(n)$ by definition, and we have the following relation: 
\begin{equation}
    \bigcup_{j\geq l}A_j \subset \Big\{\sum_{j\geq l}T_{k,j}(n) \leq \max_{j\geq l}u_j\Big\},
    \label{eqn:decomposition-2}
\end{equation} which follows from the fact that $\cup_{j\geq l}A_j = A_L\cup (A_{L-1}\cap A_L^c)\cup\ldots\cup (A_l\cap (\cup_{l< j \leq L}A_j)^c)$, and $$A_j\cap (\cup_{j^\prime > j}A_{j^\prime})^c\subset A_j\cap A_{j+1}^c \subset \{T_{k,j} \leq \max\{0,u_j-u_{j+1}\}\}.$$ Therefore, we have the following inequality: $$\sum_{j\geq l}T_{k,l}(N) \leq \max_{j \geq l}u_j + \sum_{i=\max_{j\geq l}u_j+1}^N\mathbb{I}\{E_{1,i}\cup E_{2,i}\cup E_{3,i}\}.$$ Taking the expectation in the above equality, and using \eqref{eqn:decomposition} and \eqref{eqn:decomposition-2}, we have the following result: $$\E[\sum_{j\geq l}T_{k,l}(N)] \leq \max_{j\geq l}u_j + \frac{16(L-l)\alpha}{\alpha-2},$$ which yields the result in part (i).

For part (ii), let $(k, t_l)$ be such that $k = k^*$ and $t_l < t_k^*$. Following a similar analysis as part (i) yields $\E[T_{k,l}(N)] \leq O(\log(N^\alpha))$, which implies that $\E[T_{k^*,l_{k^*}^*}(N)] = \Omega(N-\log(N))$. Therefore, since the number of samples satisfies $T_{k,l}^*(N) \geq T_{k^*,l_{k^*}^*}(N) = \Omega(N-\log(N))$, the decision $(k^*, t_l)$ is chosen at most $O(1)$ times \citep{lattimore2014bounded, cayci2017learning}.

\end{proof}

\section{Proof of Theorem \ref{thm:ucb-rc}}\label{app:ucb-rc}
The proof incorporates a variant of the regret analysis for quantized continuous decision sets given in \citep{combes2014unimodal} into the regret analysis for budget-constrained bandits presented in Appendix \ref{app:ucb-rb}.

\textbf{Step 1.} First, we bound the regret that stems from using a quantized decision set. Recall from Proposition \ref{prop:rew-opt} that $${\tt OPT}(\tau) \leq \tau\cdot r_*+O(1),$$ and $${\tt OPT}_Q(\tau) \geq \tau \max_{t\in\mathbb{T}_Q}~r_1(t),$$ where ${\tt OPT}_Q(\tau)$ is the optimal reward in the quantized decision set $\mathbb{T}_Q$, and $r_*=\max_{t\in\mathbb{T}}~r_1(t)$. Then, the regret under the UCB-RC Algorithm is bounded as follows:
\begin{align*}
    {\tt REG}_{\tt \pi^C}(\tau) &= {\tt OPT}(\tau)-\E[{\tt REW}_{\tt \pi^C}(\tau)],\\
    &={\tt OPT}(\tau)-{\tt OPT}_{Q}(\tau)+{\tt OPT}_{Q}(\tau)-\E[{\tt REW}_{\tt \pi^C}(\tau)],\\
    &\leq \tau\big(r_*-\max_{t\in\mathbb{T}_Q}~r_1(t)\big)+{\tt REG}_{{\tt \pi^C},Q}(\tau)+O(1),
\end{align*}
where $${\tt REG}_{{\tt \pi^C},Q}(\tau) = {\tt OPT}_Q(\tau)-\E[{\tt REW}_{\tt \pi^C}(\tau)],$$ is the regret under the UCB-RC Algorithm with respect to the optimal policy in the quantized decision set. By Assumption \ref{assn:ucb-rc}, we have: $$r_*-\max_{t\in\mathbb{T}_Q}~r_1(t) \leq a_2\delta^q,$$ since $|t_1^*-\arg\max_{t\in\mathbb{T}_Q}~r_1(t)| \leq \delta$. Thus, we have the following inequality:
\begin{equation}
    {\tt REG}_{\tt \pi^C}(\tau) \leq a_2\tau \delta^q + {\tt REG}_{{\tt \pi^C},Q}(\tau) + O(1).
    \label{eqn:ucb-rc-decomposition-1}
\end{equation}

\textbf{Step 2.} After we quantify the regret from using a quantized decision set, now we bound ${\tt REG}_{{\tt C},Q}(\tau)$, the regret of the UCB-RC Algorithm with respect to the optimal algorithm in the quantized decision set. We first present a variation of the decomposition in \eqref{eqn:decomposition-2}.

\begin{claim}
    Let $l^*=\arg\max_j~r_1(t_j)$, and $r_{Q}^*=\max_j~r_1(t_j)$ be the optimal reward rate in the quantized decision set $\mathbb{T}_Q$. For any $l\neq l^*$, let $\Delta_{1,l}=r_{Q}^*-r_1(t_l)$, and
    \begin{equation}
        z_l = 3\Big(z^2(\beta)\mathbb{C}_{1,l}\frac{(r_Q^*)^2}{\Delta_{1,l}}+O(1)\Big)\log(N^\alpha).
        \label{eqn:z-l}
    \end{equation}
    Then, we have the following for any $l\leq k \leq L(\delta)$:
    \begin{equation}
        \E[\sum_{j=l,j\neq l^*}^k\Delta_{1,j}T_{1,j}(N)]\leq z_l + \sum_{\substack{j=l+1\\j\neq l^*}}^k z_{j}\Big(1-\frac{\Delta_{1,j-1}}{\Delta_{1,j}}\Big)^+ + \frac{16L(\delta)\alpha}{\alpha-2}.
    \end{equation}
    \label{claim:ucb-rc}
\end{claim}

\begin{proof}
We have the following relation for any $l\leq k \leq L(\delta)$:
    \begin{equation}
        \cup_{j=l}^kA_j \subset \Big\{\sum_{j= l, j\neq l^*}^k\Delta_{1,j}T_{1,j}(N) \leq z_l + \sum_{\substack{j=l+1\\j\neq l^*}}^k z_{j}\Big(1-\frac{\Delta_{1,j-1}}{\Delta_{1,j}}\Big)^+\Big\},
    \end{equation}
    which can be proved by induction. Note that the $O(1)$ term in the RHS of \eqref{eqn:z-l} is bounded as follows: \begin{align*}
        2z(\beta)\mathbb{B}_{1,l}r_Q^*+8\big(\mathbb{K}_{1,l}+\mathbb{B}_{1,l}^2\big)\Delta_{1,l}\leq 2z(\beta)b_0r_*+8r_*(\kappa + b_0^2),
    \end{align*}
where $b_0 \geq \max_l~B_{1,l}$ and $\kappa\geq \max_l~\mathbb{K}_{1,l}$ are constants independent of $\delta$, but depend on ${\tt rad}(\mathbb{T})$, $\epsilon$ and $\mu_*$ under Assumption \ref{assn:ucb-rc}.
    
    By following the same steps as Lemma \ref{lem:sample-complexity}, the proof follows.
\end{proof}



By \eqref{eqn:ucb-rc-decomposition-1} and Proposition \ref{lem:reg-dec}, the regret under UCB-RC is bounded as follows:
\begin{equation}
    {\tt REG}_{\tt C}(\tau) \leq a_2\delta^q\tau+\mu\sum_{l\in [L(\delta)]}\E[\Delta_{1,l}T_{1,l}(2\tau/\mu_*)] + O\big(L(\delta)\big).
    \label{eqn:reg-c}
\end{equation}
\noindent where $\mu = \max_{t\in\mathbb{T}}\E[U_{1,1}(t)]$ and $\mu_*=\min_t\E[U_{1,1}(t)]$.

 Let the sets $A, B, D$ be defined as follows:
\begin{align*}
    A &= \{l^*-1,l^*,l^*+1\},\\
    B &= \{l:t_{min}+(l-1)\delta \in \mathcal{B}(t_1^*,\delta_0)\}\cap A^c, \\
    D &=[L(\delta)]\cap (A\cup B)^c,
\end{align*}
where $\mathcal{B}(x,\epsilon_0)$ denotes the ball in $\mathbb{R}$ centered at $x$ with radius $\epsilon_0>0$.

\begin{itemize}
    \item For any $l\in A\backslash\{l^*\}$, we have the following by Assumption \ref{assn:ucb-rc}: $$\Delta_{1,l} \leq r_*-r_1(t_l) \leq a_2(2\delta)^q,$$ since $r_* \geq r_Q^*$ and $\Big|t_1^*-\big(t_{min}+(l-1)\delta\big)\Big|\leq 2\delta$. Thus, we have: 
    \begin{equation}
        \sum_{l\in A\backslash\{l^*\}}\E\Big[\Delta_{1,l}T_{1,l}\Big(\frac{2\tau}{\mu_*}\Big)\Big] \leq a_2\mu\cdot \frac{2\tau}{\mu_*}(2\delta)^q.
        \label{eqn:reg-a}
    \end{equation}
    
    \item For $l\in B$, note that $$\Big| \big(t_{min}+(l^*-1)\delta\big) - \big(t_{min}+(l-1)\delta\big) \Big| \geq \delta(\big|l^*-l\big|-1\big),$$ which implies the following by Assumption \ref{assn:ucb-rc}: $$\Delta_{1,l} \geq a_1 \big||l^*-l|-1\big|^q\delta^q.$$ By using this result and Claim \ref{claim:ucb-rc}, we have the following bound:
    \begin{align*}
        \E\Big[\sum_{l\in B}\Delta_{1,l}T_{1,l}\Big(\frac{2\tau}{\mu_*}\Big)\Big] &\leq \alpha\sum_{l=1}^{L(\delta)}\frac{3z^2(\beta)\mathbb{C}_{1,l}r_*^2\log(\tau)}{a_1 (l\cdot \delta)^q} + O(\log(\tau)|B|), \\
         & \leq 3\alpha z^2(\beta)r_*^2\frac{\log(\tau)}{a_1\delta^q}\mathbb{C}^\star\sum_{l=1}^\infty\frac{1}{l^q} + O(\log(\tau)|B|),
    \end{align*}
    \noindent where $\mathbb{C}_{1,l}\leq \mathbb{C}^\star$ for all $l\in L(\delta)$ with $$\mathbb{C}^\star = \Big(\frac{\E[X_{1,1}^2]}{\mu_*^2}+\frac{\E[R_{1,1}^2]}{\epsilon^2\big(\E[R_{1,1}]\big)^2}\Big).$$ Therefore, we have the following bound:
    \begin{equation}
        \E\Big[\sum_{l\in B}\Delta_{1,l}T_{1,l}\Big(\frac{2\tau}{\mu_*}\Big)\Big] \leq 3\alpha z^2(\beta)r_*^2\frac{\log(\tau)}{a_1\delta^q}\frac{\mathbb{C}^\star q}{q-1} + O(\log(\tau)|B|).
        \label{eqn:reg-b}
    \end{equation}
    
    \item For $l\in D$, we have $\Big| \big(t_{min}+(l^*-1)\delta\big) - \big(t_{min}+(l-1)\delta\big) \Big| \geq \delta_0/2$, hence the following holds by Assumption \ref{assn:ucb-rc}: $$\Delta_{1,l} \geq a_1(\delta_0/2)^q.$$ Consequently, we have the following upper bound:
    \begin{equation}
        \E\Big[\sum_{l\in D}\Delta_{1,l}T_{1,l}\Big(\frac{2\tau}{\mu_*}\Big)\Big] \leq 3\alpha z^2(\beta)r_*^2\mathbb{C}^\star\frac{\log(\tau)L(\delta)}{a_1(\delta_0/2)^q}+O(\log(\tau)|D|).
        \label{eqn:reg-d}
    \end{equation}
\end{itemize}

Substituting the results in \eqref{eqn:reg-a}, \eqref{eqn:reg-b} and \eqref{eqn:reg-d} into \eqref{eqn:reg-c}, we obtain the following upper bound:
\begin{equation}
    {\tt REG}_{\tt \pi^C}(\tau) \leq a_2\frac{2\mu}{\mu_*}\tau(3\delta)^q + \frac{3\alpha q}{q-1}\mathbb{C}^\star z^2(\beta)r_*^2\frac{\log(\tau)}{a_1\delta^q} + O(\log(\tau)L(\delta)).
\end{equation}
With the choice $\delta^q = \sqrt{\log(\tau)/\tau}$, we have $\log(\tau)L(\delta) = o(\sqrt{\tau\log(\tau)})$. Therefore, $$\lim\sup_{\tau\rightarrow}\frac{{\tt REG}_{\tt \pi^C}(\tau)}{\sqrt{\tau\log(\tau)}} \leq 6^q a_2\frac{ \mu}{\mu_*} + \frac{3\alpha q}{a_1(q-1)}\mathbb{C}^\star z^2(\beta)r_*^2.$$

\vskip 0.2in
\bibliography{main}

\end{document}